\documentclass{article}

\usepackage{microtype}
\usepackage{graphicx}
\usepackage{subfigure}
\usepackage{booktabs} 

\usepackage{hyperref}

\usepackage[accepted]{icml2024}

\usepackage{amsmath}
\usepackage{amssymb}
\usepackage{mathtools}
\usepackage{amsthm}

\usepackage[capitalize,noabbrev]{cleveref}
\usepackage{multirow}

\newcommand{\setmid}{{\ \,\mid\ \,}}

\newcommand{\bv}{b_{V}}

\newcommand{\GnD}{\mathbb{G}_{n,D}}

\newcommand{\RR}{\mathbb{R}}
\newcommand{\NN}{\mathbb{N}}
\newcommand{\lms}{ \{\!\!\{ }
\newcommand{\rms}{ \}\!\!\} }
\newcommand{\m}{ \mid }
\newcommand{\lb}{ \!\left( }
\newcommand{\rb}{ \!\right) }

\newcommand*{\rom}[1]{\expandafter\@slowromancap\romannumeral #1@}

\newcommand{\M}{\mathbb{M}}
\newcommand{\W}{\mathbb{W}}
\newcommand{\G}{\mathcal{G}}

\newcommand{\N}{\mathcal{N}}

\newcommand{\affdim}{\mathrm{AffineDim}}
\newcommand{\zv}{\Vec{0}_{V}}

\newcommand{\bc}{\mathbf{c}}
\newcommand{\bi}{\mathbf{i}}
\newcommand{\bC}{\mathbf{C}}
\newcommand{\Cglobal}{\bC_{\mathrm{global}}}
\newcommand{\cglobal}{\bc_{\mathrm{global}}}

\newcommand{\WeLNet}{\textbf{WeLNet} }

\usepackage{amsmath}

\DeclareSymbolFont{boldoperators}{OT1}{cmr}{bx}{n}
\SetSymbolFont{boldoperators}{bold}{OT1}{cmr}{bx}{n}
\edef\bar{\unexpanded{\protect\mathaccentV{bar}}\number\symboldoperators16}

\newcommand{\bx}{\boldsymbol{x}}
\newcommand{\by}{\boldsymbol{y}}

\newcommand{\btheta}{\boldsymbol{\theta}}
\newcommand{\bth}{\btheta}

\newcommand{\of}[1]{{\left({#1}\right)}}
\newcommand{\br}[1]{{\left({#1}\right)}}

\usepackage{amsmath}
\usepackage{amssymb}
\usepackage{mathtools}
\usepackage{amsthm}

\usepackage[capitalize,noabbrev]{cleveref}
\usepackage{soul}

\usepackage{float}
\usepackage{thmtools} 
\usepackage{thm-restate}

\newcommand{\Rthree}{\RR^{3 \times n}}

\newcommand{\embed}{\text{\textbf{HASH}}}

\newcommand{\bphi}{\boldsymbol{\phi}}
\newcommand{\X}{\mathcal{X}}
\newcommand{\Y}{\mathcal{Y}}

\renewcommand{\O}{\mathcal{O}}

\newcommand{\h}{\mathbf{h}}

\newcommand{\PPGN}{\mathrm{PPGN}}

\newcommand{\ctokl}{\mathbf{C_{(1)}}(k,l)}

\newcommand{\czplain}{\mathbf{C_{(0)}}}
\newcommand{\ctone}{\mathbf{C_{(1)}}(i,j)}
\newcommand{\ctoneni}{\mathbf{C_{(1)}}}
\newcommand{\cttwo}{\mathbf{C_{(2)}}(i,j)}
\newcommand{\cttwoni}{\mathbf{C_{(2)}}}

\newcommand{\ctthreeni}{\mathbf{C_{(3)}}}
\newcommand{\ctfour}{\mathbf{C_{(4)}}(i,j)}
\newcommand{\ctfive}{\mathbf{C_{(5)}}(i,j)}
\newcommand{\cone}{\mathrm{Cone}}

\newcommand{\PPGNan}{\PPGN_{\mathrm{an}}(\theta;\Delta,T)}

\newcommand{\Xout}{X^{\mathrm{out}}}
\newcommand{\Vout}{V^{\mathrm{out}}}
\newcommand{\welconv}{\mathbf{WeLConv}}
\usepackage{nicematrix}

\theoremstyle{plain}
\newtheorem{theorem}{Theorem}[section]
\newtheorem{proposition}[theorem]{Proposition}
\newtheorem{lemma}[theorem]{Lemma}

\theoremstyle{definition}
\newtheorem{definition}[theorem]{Definition}

\theoremstyle{remark}

\usepackage[textsize=tiny]{todonotes}

\icmltitlerunning{Weisfeiler Leman for Euclidean Equivariant Machine Learning}

\begin{document}

\twocolumn[
\icmltitle{Weisfeiler Leman for Euclidean Equivariant Machine Learning}




\begin{icmlauthorlist}
\icmlauthor{Snir Hordan}{math}
\icmlauthor{Tal Amir}{math}
\icmlauthor{Nadav Dym}{math,cs}

\end{icmlauthorlist}

\icmlaffiliation{cs}{Faculty of Computer Science, Technion - Israel Institute of Technology, Haifa, Israel}
\icmlaffiliation{math}{Faculty of Mathematics, Technion - Israel Institute of Technology, Haifa, Israel}
\icmlcorrespondingauthor{Snir Hordan}{snirhordan@campus.technion.ac.il}


\icmlkeywords{Machine Learning, ICML}

\vskip 0.3in
]



\printAffiliationsAndNotice{}  

\begin{abstract}

 The $k$-Weisfeiler-Leman ($k$-WL) graph isomorphism test hierarchy is a common method for assessing the expressive power of graph neural networks (GNNs). Recently, GNNs whose expressive power is equivalent to the $2$-WL test were proven to be universal on weighted graphs which encode $3\mathrm{D}$ point cloud data, yet this result is limited to \textit{invariant} continuous functions on point clouds. In this paper, we extend this result in three ways: Firstly, we show that PPGN \cite{NEURIPS2019_provably} can simulate $2$-WL \emph{uniformly} on all point clouds with low complexity. Secondly, we show that $2$-WL tests can be extended to point clouds which include both positions and velocities, a scenario often encountered in applications. Finally, we provide a general framework for proving equivariant universality and leverage it to prove that a simple modification of this \emph{invariant} PPGN architecture can be used to obtain a universal \emph{equivariant} architecture that can approximate all continuous equivariant functions uniformly. Building on our results, we develop our \WeLNet architecture, which sets new state-of-the-art results on the N-Body dynamics task and the GEOM-QM9 molecular conformation generation task.

\end{abstract}
\section{Introduction}


Machine learning (ML) models that are equivariant to group symmetries of data have been at the focal point of recent research. Examples of equivariant models range from Convolutional Neural Networks (CNNs) that respect the translation symmetry of images, through graph neural networks (GNNs) that enforce permutation invariance to account for the invariance of the order of a node's neighbors, to models that respect symmetries of the Lornentz \cite{Lorenz} or Special Linear group \cite{lawrence2023learning}. Equivariant models are well-known to be empirically robust \cite{cohen2016group} and lead to improved generalization \cite{mircea-equi}.


In this paper, our focus will be on ML models for point clouds. A point cloud is a finite collection of points, usually in $\RR^{3}$, with the natural symmetry of invariance to permutation. Point clouds are flexible objects which are used to represent discretized surfaces, molecules, and particles \cite{discrete-cloud}. 

In many of these settings, point clouds have an additional natural symmetry to the actions of rotations, reflections, and translation, which generate the Euclidean group, which is the group of isometries of Euclidean space. Due to their applications in chemo-informatics \cite{pozdnyakov2023smooth}, particle dynamics \cite{schutt2017schnet}, and computer vision \cite{qi2017pointnet}, point cloud networks that respect Euclidean symmetries have attracted considerable attention in recent years \cite{thomas2018tensor,egnn,deng2021vector,clofnet}. 

An emerging paradigm for constructing equivariant networks for point clouds goes through the observation that an ordered set of points in Euclidean space is determined, up to Euclidean symmetry, by the set's pairwise distance matrix. Each such matrix can be identified with a complete weighted graph. Using this identification, point-cloud neural networks can be constructed by applying standard GNNs to a point cloud's distance matrix \cite{lim2022sign}, as GNNs enforce permutation invariance. 

To address the theoretical potential and limitations of these graph-based equivariant models, a recent line of research \cite{pozdnyakov2022incompleteness, hordan2023complete, rose2023iterations} seeks to assess their expressive power via $k$-WL tests \cite{weisfeiler1968reduction}, a hierarchy of graph isomorphism tests with strictly increasing distinguishing power as one goes up the hierarchy \cite{CaiFuererImmerman1992}. This hierarchy has shown to be useful in assessing the expressive power of GNNs on combinatorial graphs \cite{morris2020weisfeiler, xu2018powerful}.

For GNNs applied to point clouds, \cite{pozdnyakov2022incompleteness} showed that there exist pairs of non-isometric point clouds that cannot be distinguished by a GNN whose expressive power is bounded by $1$-WL. This suggests that the capacity of such GNNs, which include the popular Message Passing Neural Networks \cite{xu2018powerful}, is limited, and more expressive GNNs may be needed for some geometric tasks.



The geometric incompleteness of $1$-WL was proven to be remedied when climbing the $k$-WL hierarchy ladder by only a single step \cite{hordan2023complete, rose2023iterations}. That is, the 2-WL graph isomorphism test for point clouds is \emph{complete}:  it can distinguish between all non-isometric 3D point clouds. As a corollary, it can be shown that GNNs that can simulate the $2$-WL test, implemented with suitable aggregation operators, can approximate all continuous functions on point clouds invariant to Euclidean actions \cite{hordan2023complete,li2023distance}. 

These theoretical findings are coupled with strong empirical results attained by \citeauthor{li2023distance} with $2$-WL-based methods for Euclidean invariant geometric tasks. It may be argued that these recent results indicate that $2$-WL-based methods are well suited for learning on point clouds, particularly for molecular datasets and other datasets with a high degree of data symmetry \cite{pozdnyakov2020incompleteness, li2023distance}, where $1$-WL may falter. The low data dimensionality of some molecular datasets makes $ 2$-WL-based methods computationally feasible,
while overcoming the inherent limitations of models bounded in their expressive power by $1$-WL.

Nonetheless, 2-WL-based methods are less studied than other research directions in the equivariant point cloud literature, and several practical and theoretical challenges remain. In this paper, we address three such important challenges:


Firstly, the universality results stated above assume an implementation of GNNs that can simulate the $2$-WL test. While the PPGN \cite{NEURIPS2019_provably} architecture is indeed known to simulate the $2$-WL test, it is only guaranteed to separate different graph pairs using different network parameters. To date, it is not clear that $2$-WL can be simulated by a network with fixed parameters uniformly on all point clouds of size $n$. Moreover, even for pairwise separation, the time complexity in the proof in \cite{NEURIPS2019_provably} is prohibitively high: the time complexity of a PPGN block is estimated at $\mathcal{O}{(N_F\cdot n^{\omega})}$, where $n^\omega$ is the complexity of matrix multiplication and $N_F$ denotes the dimension of the edge features. In the proof in \cite{NEURIPS2019_provably}, $N_F$ grows exponentially both in the number of iterations and in the number of input features.

Secondly, the input to point-cloud tasks that originate from physical simulation is often not one, but two points clouds: one defining particle positions, the other defining particle velocity. It is a desideratum to construct an architecture that is complete with respect to such data, that is, it can distinguish among all position-velocity pairs up to symmetries.

Lastly, existing universality results for $2$-WL-based methods for point clouds are restricted to permutation- and rotation-invariant functions, or to functions that are permutation-invariant and rotation-equivariant. The case of functions that are jointly permutation- and rotation-equivariant is more difficult to characterize and has not been addressed to date. 
\subsection{Contributions}\label{sub:contributions} The three contributions of this manuscript address these three challenges:

\textbf{Contribution 1: Cardinality of $2$-WL simulation}
We show that PPGN with a fixed finite number of parameters can \emph{uniformly separate} continuous families of $2$-WL separable graphs, and in particular all $3\mathrm{D}$ point clouds. The number of parameters depends moderately on the intrinsic dimension of these continuous graph families. Consequently, the memory and runtime complexity reduces to $\mathcal{O}(n^2)$ and $\O(n^{\omega})$, respectively. This result particularly applies to weighted graphs derived from point clouds, but it is also of \textbf{independent interest} for general  graphs processed by the \textbf{most studied general $k$-WL based GNNs }\cite{NEURIPS2019_provably, morris2018weisfeiler}, as we prove separation for a general family of weighted graphs.

\textbf{Contribution 2: Combining positions and velocities} We suggest an adaptation of the $2$-WL test to the case of position-velocity pairs, and show this test is complete. These results can also be easily extended to cases where additional geometric node features such as forces, or non-geometric features such as atomic numbers, are present.

\textbf{Contribution 3: Equivariant Universality}
We propose a simple method to obtain an \emph{equivariant} architecture from the \emph{invariant} $2$-WL based PPGN architecture and show that this architecture is equivariant universal. That is, it can approximate all continuous \textit{equivariant} functions, uniformly on compact sets. 

Building on these results, we introduce our \textbf{We}isfeiler-\textbf{L}eman \textbf{Net}work architecture, \WeLNet, which can process position-velocity pairs, produce functions fully equivariant to permutations, rotations, and translation, and is provably complete and universal. 

A unique property of \textbf{WeLNet} is that for Lebesgue almost every choice of its parameters,  it \textbf{is provably complete} precisely in the settings in which it is implemented \textbf{in practice}. This is in contrast with previous complete constructions, which typically require an unrealistically large number of parameters to be provably complete, such as ClofNet \cite{clofnet}, GemNet \cite{gemnet} and TFN \cite{thomas2018tensor, Dym2020OnTU}.


Our experiments show that \WeLNet compares favorably with state of the art architectures for the $\mathrm{N}$-body physical simulation tasks and molecular conformation generation. Additionally, we empirically validate our theory by showing that PPGN can separate challenging pairs of $2$-WL separable graphs with a very small number of features per edge. This effect is especially pronounced for analytic non-polynomial activations, where a single feature per edge is provably sufficient. 
\section{Related Work}\label{app:related}

\paragraph{Equivariant Universality}
Invariant architectures that compute features that completely determine point clouds with rotation and permutation symmetries were discussed extensively in the machine learning \cite{bokman2022zz,clofnet,kurlin, DBLP:conf/cvpr/WiddowsonK23, Dkulrin_molecults} and computational chemistry \cite{doi:10.1137/15M1054183,ACE,nigam2023completeness} literature. Invariant universality is an immediate corollary \cite{dym2023low}. In contrast, \emph{equivariant universality} has not been fully addressed until this work.

\citet{Dym2020OnTU} show that the Tensor Field Network \cite{thomas2018tensor} has full equivariant universality, but this construction requires irreducible representations of arbitrarily high order. \citet{puny2021frame} provide simple and computationally sound equivariant constructions that are universal but have discontinuous singularities and assume the incomplete case of a generic point cloud (a point cloud that has a covariance matrix with distinct eigenvalues). \citet{hordan2023complete} characterize functions that are permutation invariant and rotation equivariant, but do not address the fully-equivaraint case. 
  Finally, \citet{villar} provides an implicit characterization of permutation- and rotation-equivariant functions, via permutation-invariant and rotation-equivariant functions. However, these results \emph{per se} are not enough to construct an explicit characterization of such functions, or to universally approximate them. In fact, as noted in \cite{pozdnyakov2022incompleteness}, the practical implementation of these functions proposed in \cite{villar} is not universal. Ultimately, we maintain that the problem of joined permutation- and rotation-equivariance has not been fully addressed up to this work, as well as the problem of joined position-velocity universality.

\paragraph{Simulating WL} 
The seminal works of \citeauthor{xu2018powerful,morris2018weisfeiler} have shown that sufficiently expressive message-passing neural networks (MPNNs) are equivalent to the $1$-WL test, but do not give a reasonable bound on the network size necessary to uniformly separate a large finite or infinite collection of graphs. However,  recent work has shown that $1$-WL can be simulated by ReLU MPNNs with polylogarithmic complexity in the size of the  combinatorial graphs \cite{aamand2022exponentially}, while MPNNs with analytic non-polynomial activations can achieve $1$-WL separation with low complexity independent of the graph size \cite{amir2023neural}, as long as the feature space is discrete. This expressivity gap between analytic and piecewise linear, and even piecewise-polynomial, activations, is also discussed in \citet{khalife2023power}.

GNNs simulating higher-order $k$-WL tests have been proposed in \cite{NEURIPS2019_provably,morris2018weisfeiler}, but these works have also focused on pairwise separation. \citet{jogl2024expressivity} have expanded upon the results by \cite{morris2018weisfeiler} to include simulation by MPNNs of recent GNNs via a notion of a transformed graph. Yet, a simulation by an MPNN via a transformed graph, of a test with equivalent distinguishing power to $2$-WL, would be of computational order of $O(n^4)$ (vs. our $O(n^{\omega}),  \omega \leq 3$) with the best-known results for uniformly simulating MPNNs \cite{amir2023neural}, see \cref{app:extensions}. \cite{hordan2023complete} discusses uniform separation with computational complexity only modestly higher than what we use here, $O(n^3\log(n))$ vs. $O(n^{\omega})$, with $\omega \leq 3$, but the GNN discussed there used sorting-based aggregations, which are not commonly used in practice. To the best of our knowledge, this is the first work in which popular high-order GNNs are shown to uniformly separate graphs with low feature cardinality and practical time complexity.

\section{Problem Setup}
\subsection{Mathematical Notation}\label{sub:notation}
A (finite) \emph{multiset} $\lms y_1,\ldots,y_n \rms $ is an unordered collection of elements where repetitions are allowed.

Let $\G$ be a group acting on a set $\X$. For $X,Y \in \X$,  we say that $X\cong Y$ if $Y=gX$ for some $g \in \G$. We say that a function $f:\X \to \Y$  is \emph{invariant} if $f(gx)=f(x)$ for all $x\in X, g\in G$. We say that $f$ is \emph{equivariant} if $\Y$ is also endowed with some action of $G$ and  $f(gx)=gf(x) $ for all $x\in \X,g\in \G$. 

In this paper, we consider the set $\X=\RR^{6\times n}$ and regard it as a set of pairs $(X,V)$, with $X,V \in \RR^{3\times n}$ denoting the positions and velocities of $n$ particles in $\RR^3$. We denote the $n$ columns of $X$ and $V$ by $x_i$ and $v_i$ respectively, $i=1,\ldots,n$. The natural symmetries of $(X,V)$ are permutations, rotations, and translations. Formally, we say that $(X,V)\cong (X',V')$ if there exist a permutation $\tau$, a (proper or improper) rotation $R\in O(3)$, and a `translation' vector $t$, such that 
\begin{align*}
(x_1',\ldots,x_n')&=(Rx_{\tau(1)}+t,\ldots,Rx_{\tau(n)}+t)\\
(v_1',\ldots,v_n')&=(Rv_{\tau(1)},\ldots,Rv_{\tau(n)})
\end{align*}

We consider functions $f:\RR^{6 \times n}\to \RR^{6\times n}$ (and scalar functions $f:\RR^{6 \times n}\to \RR$) that are equivariant (respectively invariant) to permutations, rotations, and translations. 

Description of additional related works, and proofs of all theorem, are given in the appendix.

\subsection{2-WL Tests}\label{sec:tests}



The $2$-WL test is a test for determining whether a given pair of (combinatorial) graphs are isomorphic (that is, related by a permutation). It is defined as follows: let $\G$ be a graph with vertices indexed by $[n] = \{1,2,\ldots,n\}$, possibly endowed with node features $w_v$ and edge features $w_{u,v}$. We denote each ordered pair of vertices by a multi-index $\mathbf{i} = \left( i_1,i_2 \right) \in [n]^2$. For each such pair $\mathbf{i}$, the $2$-WL test maintains a \emph{coloring} $\mathbf{C}(\mathbf{i})$ that belongs to a discrete set, and updates it iteratively. First, the coloring of each pair $\mathbf{i}$ is assigned an initial value  $\mathbf{C_{(0)}}(\mathbf{i})$ that encodes whether there is an edge between the paired nodes $i_1$ and $i_2$, and the value of that edge's feature $w_{i_1,i_2}$, if given as input. Node features $w_i$ are encoded in this initial coloring through pairs of identical indices $\mathbf{i}=(i,i)$. 
Then the color of each pair $\mathbf{i}$ is iteratively updated according to the colors of its `neighboring' pairs. The colors of neighboring pairs is \emph{aggregated} to an `intermediate color'  $\mathbf{\tilde{C}_{(t+1)}}(\mathbf{i}) $,
\begin{equation}
\mathbf{\tilde{C}_{(t+1)}}(\mathbf{i}) = \embed( \lms  \mathbf{C_{(\mathbf{t})}}(i_1,j),C_{(\mathbf{t})}(j,i_2)  \rms_{j=1}^{n}  ) \label{eq:2wlagg}
\end{equation}
This `intermediate color' is then \emph{combined} with the previous color at $\mathbf{i}$ to form the new color
\begin{equation}
\mathbf{C_{(t+1)}}(\mathbf{i}) = \embed\left(\mathbf{{C}_{(t)}}(\mathbf{i}),\mathbf{\tilde{C}_{(t+1)}}(\mathbf{i})   \right) \label{eq:2wlcombine}
\end{equation}

 This process is repeated $T$ times to obtain a final coloring $ \lms \mathbf{C_{(T)}}(\mathbf{i}) \rms_{\mathbf{i} \in {\left[n\right]^2}}$. A global label is then calculated by
\begin{equation}\label{eq:final-label}
    \Cglobal=\embed \left( \lms \mathbf{C_{(T)}}(\mathbf{i}) \; | \; \mathbf{i} \in [n]^2 \rms \right).
\end{equation}
 To test the isomorphism of two graphs $G,G'$, this process is run simultaneously for both graphs to produce global features $\Cglobal$ and $\Cglobal'$. The $\embed$ function is defined recursively throughout this process, such that each new input encountered is labeled with a distinct feature. If $G$ and $G'$ are isomorphic, then by the permutation invariant nature of the test, $\Cglobal=\Cglobal'$. Otherwise, for `nicely behaved' non-isomorphic graphs the final global features will be distinct, but there are pathological (highly regular) graph pairs for which identical features will be obtained. Thus, the $2$-WL test is not \emph{complete} on the class of general graphs.
 


\subsection{Geometric 2-WL Tests}
We now turn to the geometric setting. Here we are given two point clouds $X,X' \in \RR^{3\times n}$ (we will discuss including velocities later), and our goal is to devise a test to check whether they are equivalent up to permutation, rotation, and translation. As observed in e.g, \cite{egnn}, this problem can be equivalently rephrased as the problem of distinguishing between two (complete) \emph{weighted graphs} $\G(X)$ and 
$\G(X')$ whose nodes correspond to the indices of the points, and whose edge weights encode the pairwise distances $\|x_i-x_j\| $ (respectively, $\|x_i'-x_j'\|$).  The two weighted graphs $\G(X)$ and $\G(X')$ are isomorphic (that is, related by a permutation) if and only if $X \cong X'$ \cite{egnn}. Accordingly, we obtain a test to check whether $X \cong X'$ by applying the $2$-WL test to the corresponding graphs $\G(X)$ and $\G(X')$ and checking whether they yield an identical output. As mentioned earlier,  \citeauthor{rose2023iterations} showed that in the 3D geometric setting, the 2-WL test is complete. That is, the $2$-WL test will assign the same global feature to $\G(X)$ and $\G(X')$ if, and only if,  the point clouds $X,X'$ are related by permutation, rotation, and translation. 

Note that, although the $2$-WL test is typically applied to pairs of `discrete' graphs, it can easily be applied to a pair of graphs with `continuous features', since ultimately for a fixed pair of graphs the number of features is finite. The main challenge in the continuous feature case is proposing practical, differentiable, graph neural networks which can simulate the $2$-WL test. The involves  replacing the $\embed$ functions with functions that are both differentiable and injective on an (uncountably) infinite and continuous feature space. These issues are addressed in our second contribution on simulating $2$-WL tests with differentiable models (Section \ref{sec:ppgnblock}).

In the three following sections, we will address the three challenges outlined in \cref{sub:contributions}. In \cref{sec:ppgnblock} we show that the PPGN architecture can simulate 2-WL tests, even with a continuum of features, with relatively low complexity. In \cref{sec:v} we discuss how to define 2-WL tests that are complete when applied to position-velocity pairs $(X,V)$. In \cref{sec:equi} we show that the PPGN architecture, combined with an appropriate pooling operation, is a universal approximator of continuous functions that are equivariant to permutations, rotations and translations.

\subsection{Extensions and Limitations}

There are many variants of the setting above which could be considered: $O(d)$ equivariance with $d\neq 3$, allowing only proper rotations in $SO(d)$ rather than all rotations, and allowing multiple equivariant features per node rather than just a position-velocity pair. In Appendix \ref{app:extensions} we outline how our approach can be extended to these scenarios.

Our universality results hold only for complete distance matrices and not for geometric graphs with a notion of a local neighborhood. Often in applications, a distance threshold is used to allow for better complexity. Thetheoretical results presented in this manuscript cannot be directly applied to this setting, though the \WeLNet architecture can be naturally adapted to these cases.

\section{Simulation of $2$-WL with Exponentially Lower Complexity}\label{sec:ppgnblock}

In this section we discuss our first contribution regarding designing neural networks which simulate the $2$-WL test. 

Models that simulate the $2$-WL test  replace the  $\embed$ functions the combinatorial test uses, which are defined anew for every graph pair, with differentiable functions which are globally defined on all graphs. 

The three main models proposed in the literature for simulating $2$-WL (equivalently, the $3$-OWL test, see \cite{morris2021weisfeiler}) are the $3$-GNN model from \cite{morris2018weisfeiler}, and the $3$-EGN and PPGN models from \cite{NEURIPS2019_provably}. Following \citeauthor{li2023distance}, we focus on PPGN in our analysis and implementation, as this model is more efficient than the rest due to an elegant usage  of matrix multiplications.

We start by describing the PPGN model. Then we explain in what sense existing results have shown that it simulates the $2$-WL test, and explain the shortcomings of these previous results. We then provide new, significantly stronger separation results. We note that this section is relevant for any choice of continuous or discrete labeling used to initialize $2$-WL, and thus should be of interest to graph neural network research also beyond the scope of its applications to Euclidean point clouds.
\subsection{PPGN architecture}\label{sub:ppgn}
The input to the PPGN architecture is the same as to the $2$-WL test, that is, the same collection of pairwise features  
\begin{equation}\label{eq_PPGN_init}
   \bc_{(0)}(\bi)=\bC_{(0)}(\bi) \text{ where } \bi=(i_1,i_2)\in [n]^2
\end{equation}
obtained from the input graph $\G$. We will assume that all graphs have $n$ vertices and their edge and node features are $D$ dimensional. We denote the collection of all such graphs by $\GnD$, 

Similarly to $2$-WL, PPGN  iteratively defines new pairwise features $\bc_{(t+1)}$ from the previous features  $\bc_{(t)}$ 
using an aggregation and combination step, the only difference being that now the $\embed$ functions are replaced by differentiable functions. Specifically, the aggregation function used by PPGN involves two MLPs $\phi^{(1,t)}$ and $\phi^{(2,t)}$:
\begin{equation}\label{eq:PPGNagg}
\tilde{\bc}_{(t+1)}(\bi)=\sum_{j=1}^n \phi^{(1,t)}\left(\bc_{(t)}(i_1,j) \right) \odot \phi^{(2,t)}\left(\bc_{(t)}(j,i_2) \right).
\end{equation}
Here the output of the two MLPs has the same dimension, which we denote by $D_{(t+1)}$, and $\odot$ denotes the entrywise (Hadamard) product. Note that \eqref{eq:PPGNagg} can be implemented  by independently computing $D_{(t+1)}$ matrix products.

We note that \eqref{eq:PPGNagg}, which corresponds to \eqref{eq:2wlagg} in the combinatorial case, is a well-defined function on the multiset $ \lms  (\mathbf{c_{(\mathbf{t})}}(i_1,j),\bc_{(\mathbf{t})}(j,i_2)  \rms_{j=1}^{n}$; that is, permuting the $j$ index will not affect the result.

The combination step in PPGN involves a third MLP $\phi^{(3,t)}$, whose output dimension is also $D_{(t+1)}$:
\begin{equation}\label{eq_PPGN_update_2}
\bc_{(t+1)}(\bi)=\tilde{\bc}_{(t+1)}(\bi) \odot \phi^{(3,t)}\left(\bc_{(t)}(\bi)\right).
\end{equation}
We note that in this choice of the combination step, we follow \citeauthor{li2023distance}. This product-based step is more computationally efficient than the original concatenation-based combination step of \citeauthor{NEURIPS2019_provably}. We address the simpler concatenation-based combination step in Appendix A.

After $T$ iterations, a graph-level representation $\cglobal$ is computed via a `readout' function that operates on the multiset of all $T$-level features $ \lms\bc_{(T)}(\bi)\rms_{\bi\in[n]^2}$. This is done using a final MLP, denoted $\phi_{\text{READOUT}}$, via 
 $$\cglobal=\sum_{\bi\in[n]^2} \phi_{\text{READOUT}}(\bc_{(T)}(\bi)).$$

 \paragraph{Analytic PPGN} The PPGN architecture implicitly depends on several components. In our analysis in the next subsection we will focus on a simple instantiation, where all intermediate dimensions are equal to the same number $\Delta$, that is $D_{(t)}=\Delta$ for $t=0,\ldots,T-1$, and the MLPs $\phi^{(s,t)}$, $s=1,2,3$ and $\phi_{\text{READOUT}}$ are shallow networks of the form $\rho(Ax+b)$, with $A \in \RR^{\Delta \times \Delta}$, $b \in \RR^{\Delta}$, and $\rho : \RR \to \RR$ being an analytic non-polynomial function applied element-wise. This includes common activations, such as tanh, silu, sigmoid and most other smooth activation functions, but does not include piecewise-linear activations such as ReLU and leaky ReLU (for more on this see Figure \ref{fig:sep}). Under these assumptions, a PPGN network is completely determined by the number of nodes $n$, input feature dimension $D$, hidden feature dimension $\Delta$, the number of iterations $T$, and the parameters of all the linear layers in the MLPs $\phi^{(s,t)}$ and $\phi_{\text{READOUT}}$,  which we denote by $\theta$. We call a PPGN network satisfying all these assumption an \emph{analytic PPGN network}, and denote it by $\PPGNan$. 

\subsection{PPGN separation}
In \cite{NEURIPS2019_provably} it is proven that PPGN simulates $2$-WL in the following sense: firstly, by construction, if $\G $ and $ \G'$ cannot be separated by $2$-WL, then they cannot be separated by PPGN either. Conversely, if $\G $ and $ \G'$ represent graphs that are separated by $2$-WL, then PPGN with  sufficiently large MLPs $\bphi$ will separate $\G$ and $ \G'$. 

This result has two limitations. The first is that the size of the PPGN networks in the separation proof provided in \cite{NEURIPS2019_provably} is extremely large. Their construction relies on a power-sum polynomial construction whose cardinality depends exponentially on the input and number of 2-WL iterations $T$, coupled with an approximation of the polynomials by MLP --- which leads to an even higher complexity. The second limitation is that the separation results obtained in \cite{NEURIPS2019_provably} are not uniform, but apply only to pairs of graphs. While this can be easily extended to uniform separation on all pairs of $2$-WL separable graphs coming from a finite family (see \cite{chen2019equivalence}), this is not the case for infinite families of graphs. Indeed, in our geometric setting, we would like to find a PPGN network of finite size and fixed parameters, that can separate \emph{all} weighted graphs generated by $(X,V)$ pairs. This is an infinite, $6n$-dimensional family of weighted graphs.

In this paper we resolve both of these limitations. We first show that the cardinality required for pairwise separation is actually extremely small: analytic PPGN require only one-dimensional features for pairwise separation.

\begin{restatable}{theorem}{pairsep}[2-WL pairwise separation]\label{thm:pairs}
    Let $(n,D,T)\in \NN^3$ and set $\Delta=1$. Let $\G,\G'\in \GnD$ represent two graphs separable by $T$ iterations of $2$-WL.  Then for Lebesgue almost every choice of the parameters $\theta$, the features $\cglobal$ and $\cglobal'$ obtained from applying $\PPGNan$ to $\G$ and $\G'$, repsectively, satisfy that $\cglobal\neq \cglobal'$.
\end{restatable}

Next, we consider the issue of uniform separation. We assume that we are given a continuous family of graphs $\X$ in $\GnD$. Equivalently, by identifying graphs with the pairwise features derived from them, we can think of this as a family of tensors in $\RR^{n \times n \times D}$. The size of MLPs we require in this case to guarantee uniform separation on all graphs in  $\X$ will be $2\dim(\X)+1$. In particular, if $\X$ is the collection of weighted graphs that represent distance matrices of $n$ points in $\RR^3$, then $\X$ is of dimension $\dim(\RR^{3\times n})=3n $ and we will require PPGN networks with features of dimension $D_{(t)}=6n+1 $ for separation. If we consider all position-velocity pairs $X,V$, then the dimension of $\X$ will be $6n$, and thus the dimension required for separation will be $12n+1$. Finally, under the common assumption that in problems of interest, the domain $\X$ is some manifold of low intrinsic dimension $d$, then the feature dimension required for uniform separation on $\X$ will be just $2d+1$. The full statement of our theorem for uniform separation is:   

\begin{restatable}{theorem}{uniformwl}[uniform 2-WL separation]\label{thm:uniform}
Let $(n,D,T)\in \NN^3$. Let $\X \subseteq \GnD$ be a $\sigma$-subanalytic set of dimension $d$ and set $\Delta=2d+1$. Then for Lebesgue almost every $\theta$  we have that $\G, \G' \in \X$ can be separated by $T$ iterations of 2-WL if and only if 
$\cglobal\neq \cglobal' $, where $\cglobal$ and $\cglobal'$ are obtained by applying $\PPGNan$ to $\G$ and $\G'$, respectively.
\end{restatable}
In this theorem, we identify graphs $\G\in \X$ with the $n\times n \times D$ adjacency tensors describing them. A full formal definition of a $\sigma$-subanalytic set of $\RR^{n \times n \times D}$ is beyond the scope of this paper, and can be found in \cite{amir2023neural}. For our purposes, we note that this is a rather large class of sets, which includes sets defined by analytic and polynomial constraints, and their countable unions, as well as images of such sets under polynomial, semi-algebraic, and analytic functions. In particular Euclidean spaces like $\RR^{6\times n}$ are $\sigma$-subanalytic and their image under semi-algebraic maps, such as the map that takes $(X,V)$ to the graph $\G(X,V)$ weighted by its distances, will also be a $\sigma$-subanalytic set of dimension $\leq 6n$. 

\begin{proof}[Proof idea for Theorem \ref{thm:pairs} and Theorem \ref{thm:uniform} ] The proof of pairwise separation uses three steps. First,  we  show that at every layer, pairwise separation can be achieved  via aggregations of the form of \cref{eq:PPGNagg} with arbitrarily wide neural networks $\phi^{(1,t)}, \phi^{(2,t)} $, using density arguments. Next, since wide networks are linear combinations of shallow networks, it follows that there exists  scalar networks $\phi^{(1,t)}, \phi^{(2,t)} $ which achieves pairwise separation at every layer. Lastly, the analyticity of the network implies that this separation is in fact achieved almost everywhere at every layer, which then implies that pairwise separation can be achieved with almost all parameters across all layers simultaneously. 

For uniform separation, we use the finite witness theorem from \cite{amir2023neural} that essentially claims that pairwise separating analytic functions can be extended to uniformly separating functions by taking $2d+1$ copies of the functions (with independently selected parameters). The independence of the dimension throughout the construction on the depth $T$ of the PPGN network is obtained using ideas from \cite{hordan2023complete}.     
\end{proof}

\begin{figure}[t]
\vskip 0.2in

\begin{center}

\includegraphics[width=\columnwidth]{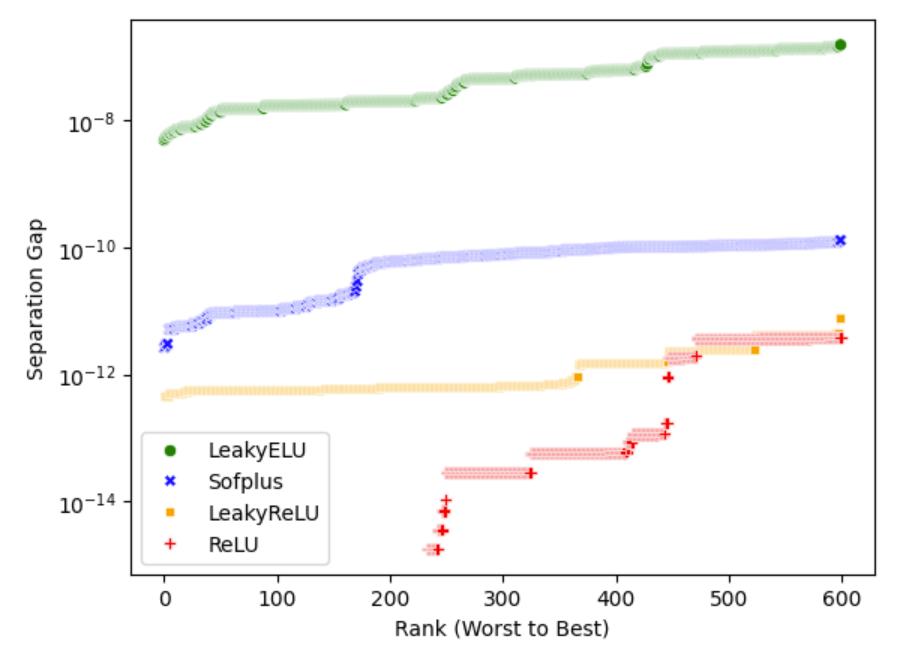}

\caption{Separation results for PPGN on the EXP dataset with a single neuron per node. As our theorem predicts, the Sofplus and leaky ELU activations which are analtyic can separate all graph pairs with a single neuron per node-pair. The separation gap is the norm of the difference between the representations of each pair of graphs, measuring how distinct they are.  Non-analytic ReLU and leaky ReLU activations yield consistently diminished separation in comparison with their analytic counterparts.}
    
\end{center}
\vskip -0.2in
\label{fig:sep}
\end{figure}





\subsection{Complexity}
The complexity of PPGN is dependent on the output dimension of each update step and the complexity of matrix multiplication. \cref{thm:uniform} requires an output dimension which is only linearly dependent on the intrinsic dimension of size $d \ll n$, thus the computational complexity is that of matrix multiplication, that is $\mathcal{O}(n^{\omega})$, where, with naive implementation, $\omega=3$. Yet GPUs are especially adept at efficiently performing matrix multiplication, and, using the Strassen algorithm, this exponent can be reduced to $\omega=2.81$ \cite{strassen}.

\subsection{Empirical Evaluation}
We empirically evaluate our claims via a separation experiment on combinatorial graphs that are $2$-WL distinguishable, yet are $1$-WL indistinguishable, via the EXP \cite{Abboud2020TheSP} dataset. It contains 600 $1$-WL equivalent graphs that can be distinguished by $2$-WL. We evaluate the \textit{separation} power of PPGN models with a random initialization and a \emph{single neuron} per node-pair on this dataset, as a function of the activation used. We ran the experiment with four different activations: the piecewise linear ReLU and LeakyReLU activations, and two roughly corresponding analytic activations: softplus and LeakyELU.  

The results of the experiment are depicted in Figure \ref{fig:sep}, which shows the difference in the global features computed by PPGN for each of the 600 graph pairs and four activation types. The results  show that the two analytic activations (as well as leakyReLU) succeeded in separating all graph pairs as predicted by our theory,  but ReLU activations did not separate all graphs. We do find that with 15 features per node-pair ReLU too is able to attain perfect separation. We also note that the two analytic activations attained better separation gaps  than the corresponding non-analytic activations. Finally, we note that in  some cases,  the separation attained is rather minor, with the difference between global features being as low as $10^{-12}$ .

\section{$2$-WL for Position-Velocity Pairs}\label{sec:v}
In this section we describe our second contribution. We consider the setting of particle dynamics tasks, in which the input is not only the particle positions $X$ but also the velocities $V$. In this setting, we define a weighted graph $\G(X,V)$ and prove that the $2$-WL test applied to such graphs is \emph{complete}. A naive application of $2$-WL to the velocities and positions separately would be insufficient as it does not guarantee a \textit{shared} rotation and permutation that relates the two. 

Therefore, the edge weights $w_{i,j}$ of the weighted graph $\G(X,V)$ will consist of the $4\times 4$ pairwise distance matrix of the vectors $x_i,x_j,v_i,v_j$ \emph{after centralizing $X$}, that is
\begin{equation}\label{eq:XVweights}
w_{i,j}=\mathrm{dist}\left(x_i-\frac{1}{n}\sum_{k=1}^n x_k,x_j-\frac{1}{n}
\sum_{k=1}^n x_k,v_i,v_j \right).
\end{equation}
Note that this edge feature is invariant to rotation and translation. Additionally, since translation does not affect velocity, we add the norms of the velocity vectors $v_i$ as node features $w_i=\|v_i\|$. 

We prove that the $2$-WL test applied to $\G(X,V)$ with the node and edge features induced from $(X,V)$, is  complete with respect to the action of permutation, rotation and translation defined in Subsection \ref{sub:notation}:
\begin{restatable}{theorem}{twoclouds}\label{thm:two-clouds}
    Let $X,V,X',V' \in \RR^{3 \times n}$. Let $\cglobal$ and $\cglobal'$ be the global features obtained from applying three iterations of the $2$-WL test to $\G(X,V)$ and $\G(X',V')$, respectively.  Then 
    $$\cglobal=\cglobal' \text{ if and only if } (X,V)\cong (X',V'). $$
    \end{restatable}

\textit{Proof Sketch.} 
The proof is based on a careful adaptation of the completeness proof in \cite{rose2023iterations}. The original proof, which only considered position inputs $X\in \RR^{3\times n}$, reconstructs $X$ from its $2$-WL coloring, up to equivalence, in a two-step process: first, three `good' (centralized) points $x_i,x_j,x_k$ are reconstructed, and then the rest of $X$ is reconstructed from the coloring of the pairs containing these three points.  In our proof we show that a similar argument can be made in the $(X,V)$ setting, where now the three `good'  points could be either velocity or (centralized) position vectors, e.g. $x_i,x_j,v_k $, and they can be used to reconstruct both $X$ and $V$, up to equivalence. 
\qed{}

\section{Equivariant Universality on Point Clouds}\label{sec:equi}
In this section we discuss our third and final contribution, which is the construction of invariant and equivariant models for position-velocity pairs. 

We note that  the completeness of the $2$-WL test for position-velocity pairs, combined with our ability to simulate $2$-WL tests uniformly with analytic PPGNs, immediately implies that function of the form $\mathrm{MLP}\circ \cglobal(X,V) $, 
can approximate all continuous \emph{invariant} functions. Here MLP stands for a Multi-Layer-Perceptron, and $\cglobal(X,V) $ is the global feature obtained from applying PPGN, with the hyperparameter configurations detailed in the previous section, to $(X,V)$. Similar universality results can be obtained for the permutation invariant and rotation equivariant case \cite{hordan2023complete}.




Our goal is to address the more challenging case of universality for permutation, rotation and translation \emph{equivariant} functions $f:\RR^{6 \times n} \to \RR^{6 \times n}$ rather than invariant functions.  To achieve equivariant universality, we will define a simple pooling mechanism that enables obtaining node-level rotation, translation and permutation equivariant features $\Xout $ and rotation and permutation equivariant and translation invariant 
 $\Vout$  from the rotation and translation invariant features of PPGN, and the input $(X,V) $. This step involves six MLPs $\psi_{(1)},\ldots,\psi_{(6)} $ and is defined as
\begin{align}\label{eq:eqpoolx}
	x_{i}^{\mathrm{out}}= x_{i}&+ \psi_{(1)}(\bc_{(T)}(i,i))v_i\\ &+ \sum_{k} \psi_{(2)}(\bc_{(T)}(i,k))(x_k - x_i) \nonumber\\
    &+ \sum_{k\neq i} \psi_{(3)}(\bc_{(T)}(i,k)) v_k \nonumber
\end{align}
\begin{align}\label{eq:eqpoolv}
v_{i}^{\mathrm{out}}= & \psi_{(4)}(\bc_{(T)}(i,i))v_i \\ 
 &+ \sum_{k} \psi_{(5)}(\bc_{(T)}(i,k))(x_k - x_i) \nonumber\\
&+	 \sum_{k\neq i} \psi_{(6)}(\bc_{(T)}(i,k)) v_k \nonumber
\end{align}

The different MLPs correspond to different equivalence classes of permutations \cite{invandequign} and to the different point clouds. The main result of this section is that the equivariant pooling layer defined above yields a universal equivariant architecture for position-velocity pairs $(X, V) $.
\begin{restatable}{theorem}{universaltwo}\label{thm:equi-uni}
	Let $\epsilon>0$. Let $\Psi: \RR^{6 \times n} \to \RR^{6 \times n}$ be a continuous permutation, rotation, and translation equivariant function. Denote $(\Xout,\Vout)=\Psi(X,V) $. Then $\Psi$ can be approximated  to $\epsilon$ accuracy on compact sets in $\RR^{6\times n}$ via the composition of the equivariant pooling layers defined in \eqref{eq:eqpoolx} and \eqref{eq:eqpoolv} with the features $\bc_{(T)}(i,k)$ obtained from  $\PPGNan$ iterations applied to $\G(X,V)$, with $T=5, \Delta=12n+1$ and appropriate parameters $\theta$.
\end{restatable}
\begin{proof}[Proof overview]

This proof consists of two main steps which we believe are of independent interest. The first step provides a characterization of \emph{polynomial} functions $f:\RR^{6\times n}\to \RR^{6\times n}$ which are permutation, rotation and translation equivariant, in terms of expressions as in \eqref{eq:eqpoolx}-\eqref{eq:eqpoolv}, but where the functions $\psi_{(j)}\circ \bc_{(T)}(i,k)$  are replaced by polynomials $p_{i,k} $ with the same equivariant structure. This result gives a more explicit characterization of equivariant polynomials than the one in \cite{villar}, and includes velocities and not only positions. 

The second step of the proof shows that the $p_{i,k}$ polynomials can be approximated by the $\psi_{(j)}\circ \bc_{(T)}(i,k)$ functions. Thus every equivariant polynomial, and more generally any continuous equivariant functions, can be approximated by expressions as in \eqref{eq:eqpoolx}-\eqref{eq:eqpoolv}.

\end{proof}	

\subsection{General Framework for Equivariant Universality}

The approach presented for proving equivariant universality is \textit{not} limited to the $2$-WL test. Equivariant universality can be proven for any procedure that replaces the $\bc_{(T)}(i,k)$ features by injective, invariant embeddings of $((x_i, v_i), (x_k, v_k), \lms (x_j, v_j) \; | \; j\neq i,k \rms)$ and performes the pooling operations \eqref{eq:eqpoolx}-\eqref{eq:eqpoolv}. Moreover, equivariant universality only for functions on the position point clouds can be attained, as well. For more details, see \cref{sec:genfram}.


\section{WeLNet} 
To summarize, we've derived a model that is equivariant to the action of permutations, rotations, and translations on position velocity pairs $(X,V)$ This model, which we name \WeLNet, employs the following steps
\begin{enumerate}
	\item Encode $(X,V)$ as a weighted graph $\G(X,V)$ as defined in \eqref{eq:XVweights}. 
	\item Apply $\PPGN$ to this weighted graph.
	\item Apply the equivariant pooling mechanism defined in Equations \eqref{eq:eqpoolx}-\eqref{eq:eqpoolv} to obtain the architecture output $(\Xout,\Vout) $. 
\end{enumerate}
We have proven that this architecture is universal when PPGN is used for five iterations, and the internal MLPs in the PPGN architecture are shallow MLPs with analytic non-polynomial activations whose feature dimension can be as small as $12n+1$. Further implementation details are described in Appendix \ref{app:exp}.

\section{Experiments}
In this section, we experimentally evaluate the performance of \WeLNet on equivariant tasks. Full details on the experiments can be found in Appendix \ref{app:exp}. \footnote{Code is available at \url{https://www.github.com/IntelliFinder/welnet}}
\subsection{N-Body Problem}
The N-body problem is a classic problem in the physical sciences, in which the model has to predict the trajectory of bodies in Euclidean space based on their initial position, physical properties (e.g. charge), and initial velocity. We test our model on the N-body dynamics prediction task \cite{egnn}, a highly popular dataset that is a standard benchmark for Euclidean equivariant models. We find that \textbf{WeLNet achieves a new state-of-the-art result}. Results are shown in Table \ref{tab:test_mse}. 

In physical systems, there may be external forces present that act independently on the particles. Therefore, we also test our model on the N-Body problem with the natural external force fields of Gravity and a Lorentz-like force (a magnetic field.) We then compare the performance of \textbf{WeLNet} to baselines that were designed for such tasks, such as ClofNet \cite{clofnet}. We find that WeLNet has significantly better results with the gravitational force and comparable results with the Lorentz-like force field. Results are shown in \cref{force-results}.

\begin{table}[ht]
  \centering
  \caption{Test MSE for the N-body dynamics prediction task. We compare to the previous SOTA, CGENN \cite{ruhe2024clifford}, Transformer-PS \cite{ps}, and other baselines.}
  \label{tab:test_mse}

    \begin{sc}
    
  \begin{tabular}{c|c}
    \hline
    Method & MSE \\
    \hline
    Linear & 0.0819 \\
    SE(3) Transformer & 0.0244 \\
    TFN & 0.0155 \\
    GNN & 0.0107 \\
    Radial Field & 0.0104 \\
    EGNN & 0.0071 \\
    ClofNet  & 0.0065 \\
    FA-GNN  & 0.0057 $\pm$ 0.0002 \\
    CN-GNN  & 0.0043 $\pm$ 0.0001\\
    SEGNN  & 0.0043 $\pm$ 0.0002\\
    MC-EGNN-2  &0.0041 $\pm$ 0.0006 \\
    Transformer-PS  & 0.0040 $\pm$ 0.00001\\
    CGENN  & 0.0039 $\pm$ 0.0001\\
    \textbf{WeLNet} (Ours) & \textbf{0.0036 $\pm$ 0.0002} \\
    \hline
  \end{tabular}
  
  \end{sc}

\end{table}

\begin{table}[t]
\caption{Test MSE on the N-Body dynamic prediction task with Gravitational force and Lorentz-like force.}
\label{force-results}
\vskip 0.15in
\begin{center}
\begin{sc}
\begin{small}
\begin{tabular}{lccr}
\toprule
Method & Gravity & Lorentz \\
\midrule
GNN    & 0.0121 & 0.02755  \\

EGNN & 0.0906& 0.032\\

ClofNet    & 0.0082  $\pm$ 0.0003& 0.0265$\pm$ 0.0004&         \\
MC-EGNN    & 0.0073$\pm$ 0.0002& \textbf{0.0240$\pm$ 0.0010} \\

WeLNet (Ours)     & \textbf{0.0054 $\pm$ 0.0001} & \textbf{0.0238$\pm$ 0.0002}\\
\bottomrule
\end{tabular}
\end{small}
\end{sc}
\end{center}
\vskip -0.1in
\end{table}

\subsection{Conformation Generation }
Generating valid molecular conformations from a molecular graph has recently become a popular task due to the rapid progress in Generative ML research \cite{DBLP:conf/nips/LuoSXT21, zhou2023unimol}. We test the ability of WeLNet to generate the 3D positions of a conformation of a molecule from its molecular graph (which only has molecular features and adjacency information) via a generative process such that it reliably approximates a reference conformation. We find that \textbf{WeLNet} achieves competitive results on the COV target and obtains \textbf{a new state of the art (SOTA) result on the MAT target with an improvement of over $\mathbf{15 \%}$} from the previous SOTA.

\begin{table}[H]
  \centering
  \caption{ MAT and COV scores on the GEOM-QM9 dataset. Best results are in bold and second best in red. We compare to the previous SOTA, UniMol \cite{zhou2023unimol} and other baselines.}
  \label{tab:geom-qm9}

\begin{tabular}{|c|c|c|c|c|}
    \hline 
\multirow{2}{*}{Model} & \multicolumn{2}{c|}{COV $(\%) \uparrow$} & \multicolumn{2}{c|}{MAT ($\mathring{\textrm{A}})\downarrow$}\\
        \cline{2-5}
        & Mean & Median & Mean & Median \\
\hline
RDKit & 83.26 & 90.78 & 0.3447 & 0.2935 \\ 
\hline
CVGAE & 0.09 & 0 & 1.6713 & 1.6088 \\ 
\hline
GraphDG & 73.33 & 84.21 & 0.4245 & 0.3973 \\ 
\hline
CGCF & 78.05 & 82.48 & 0.4219 & 0.39 \\ 
\hline
ConfVAE & 80.42 & 85.31 & 0.4066 & 0.3891 \\ 
\hline
ConfGF & 88.49 & 94.13 & 0.2673 & 0.2685 \\ 
\hline
GeoMol & 71.26 & 72 & 0.3731 & 0.3731 \\ 
\hline
DGSM & 91.49 & 95.92 & 0.2139 & 0.2137 \\ 
\hline
ClofNet & 90.21 & 93.14 & 0.2430 & 0.2457 \\

\hline

GeoDiff & {92.65} & {95.75} & {0.2016} & {0.2006} \\ 
\hline
DMCG & \color{red}{94.98} & \color{red}{98.47} & {0.2365} & {0.2312} \\
\hline
UniMol & \textbf{97.00} & \textbf{100.00} &  \color{red}{0.1907} & \color{red}{0.1754}\\
\hline
\textbf{WeLNet} (Ours)  & {92.66}& {95.29}  & \textbf{0.1614} & \textbf{0.1566}\\
\hline

\end{tabular}

\end{table}

\section{Conclusion and Future Work}
This manuscript has addressed three major challenges in the application of $2$-WL to point clouds: proof of a uniform injective simulation of $2$-WL on point clouds via the PPGN architecture, proof of the completeness of $2$-WL for positions and velocities under joint symmetries, and achievement of equivariant universality via $2$-WL and a simple pooling operator. These contributions are backed up by experiments demonstrating the practical efficacy of WeLNet, an architecture based upon $2$-WL that is complete in practice. 



For WeLNet to be provably complete, it is required to simulate $2$-WL, and this inevitably comes with a non-trivial computational cost. We may relax the completeness guarantee and consider a distance matrix with a distance cutoff and use a sparse variant of $2$-WL for improved running time and perhaps generalization. It is the subject of future work to determine whether such a relaxation can be made while maintaining strong empirical results and a notion of geometric completeness.


\section*{Acknowledgements}
ND and SH are supported by Israeli Science
Foundation grant no. 272/23. SH would like to thank the Applied Mathematics Department at the Technion for supporting this research via its `Scholarship for Excellence'.

\section*{Impact Statement}
This paper presents work whose goal is to advance the field of Machine Learning. There are many potential societal consequences of our work, none which we feel must be specifically highlighted here.
\bibliography{example_paper}
\bibliographystyle{icml2024}

\newpage
\appendix
\onecolumn
\section{Related Work and Extensions}\label{app:extensions}
\subsection{Related Work}




\paragraph{$k$-WL in the Point-Cloud Setting} Initial research on completeness on point clouds has provided Lipschitz  continuous, polynomial time algorithms for distinguishing non-isometric point clouds \cite{Kurlin2022PolynomialTimeAF, Kurlin2023SimplexwiseDD} with applications in ML \cite{Balasingham2022MaterialPP, SchwalbeKoda2023InorganicSM}. These algorithms produce complete invariants are
represented as a ‘multi-set of multi-sets’ and thus do not allow for gradient-descent-based optimization. 

The $k$-WL hierarchy has been initially used as a theoretical tool to assess the expressive power of GNNs for combinatorial graphs \cite{xu2018powerful, morris2018weisfeiler}.  Recently, the expressive power of GNNs on point clouds has been evaluated via these tests.  \citet{pozdnyakov2022incompleteness} showed that $1$-WL tests are not complete when applied to point clouds. \citet{joshi2022expressive} addressed $1$-WL tests with additional combinatorial edge features and possibly equivariant as well as invariant features. 

In contrast to $1$-WL, the  $2$-WL test \emph{is} complete when applied to 3D point clouds \cite{hordan2023complete,rose2023iterations}. This inspired GNNs \cite{li2023distance} that obtain strong empirical results on real-world molecular datasets. 
However, despite the impressive performance of \cite{li2023distance}, the empirical evaluations are limited to \textit{invariant} tasks. To the best of our knowledge, there has not been a $2$-WL based \emph{equivariant} GNN that exhibits strong performance on real-world tasks before this manuscript.

There are also variants of graph isomorphism tests that achieve completeness. For instance, it has been recently shown that subgraph GNNs \cite{zhang2021nested, frasca2022understanding} can also achieve invariant completeness \cite{li2024completeness}. However, this result does not lend itself to an equivariant architecture. Another approach to obtaining an equivariant architecture is via canonicalization, where a representative of each equivalence class under the symmetry relation is chosen and is then processed by a non-invariant architecture. Unweighted Frame Averaging \cite{puny2021frame} is a known canonicalization approach, yet it has been recently been proven that this approach cannot provably produce continuous functions \cite{dym2024equivariant}. Weighted frame averaging \cite{pozdnyakov2023smooth} can efficiently perform canonicalization and is provably continuous \cite{pozdnyakov2023smooth,dym2024equivariant}. However, the model in \cite{pozdnyakov2023smooth} assumes that there
exists a lower bound to the distance between pairs of points, which is not the case for the general case of point clouds (which allows repetitions of points).



GNNs simulating higher-order $k$-WL tests have been proposed in \cite{NEURIPS2019_provably,morris2018weisfeiler}, but these works too have focused on pairwise separation. \cite{hordan2023complete} discusses uniform separation with computational complexity only modestly higher than what we use here, $O(n^4\log(n))$ vs. $O(n^{\omega})$, with $\omega \leq 3$, but the GNN discussed there used sorting-based aggregations, which are not commonly used in practice. To the best of our knowledge, this is the first work in which popular high-order GNNs are shown to uniformly separate graphs with low feature cardinality. 

\paragraph{Simulation via a Transformed Graph.} Our definition of 2-WL is often referred to in the literature as 2-Folklore-WL, which is a test equivalent in its expressive power to 3-WL as defined in the results of \cite{jogl2024expressivity}, see \cite{CaiFuererImmerman1992}. That is, two non-isomorphic graphs can be distinguished by 2-FWL if and only if they can be distinguished by 3-WL. A transformed graph for 3-WL is a graph where each node corresponds to a 3-tuple of indices and an edge connects two such nodes if their corresponding 3-tuples differ only by a single entry, that is they are 'neighbors' in the definition of 3-WL.

The simulation of MPNN on this transformed graph corresponding to 3-WL (as defined in \cite{jogl2024expressivity}, consists of message-passing layers in which each message-passing layer computes $O(n^3)$ aggregations, corresponding to the $n^3$ nodes, and each such aggregation is of $n$ vector-valued elements, corresponding to the neighbors of each node in the transformed graph.

Applying the Finite Witness Theorem from \cite{amir2023neural} (the tool we used in our proof of Theorem 4.2) and following the analysis from \cite{amir2023neural, hordan2023complete} regarding the complexity of uniformly simulating MPNNs, yields that to obtain uniform separation as in Theorem 4.2, this MPNN simulation would require a computational complexity of $O(n^4)$. This is due to an inherent bottleneck in the MPNN approach, which is the $O(n^3)$ injective embeddings of multi-sets of $n$ elements, and each such embedding requires, by using the best-known complexity of achieving multiset injectivity via MLPs, $O(n)$ complexity by Theorem 3.3 in \cite{amir2023neural}.

In contrast, we prove (\cref{thm:uniform}) that PPGN can uniformly simulate an equally expressive test with a complexity of only $O(n^w)$, where $w$ is the exponent denoting the complexity of matrix multiplication, which can be as low as $w=2.81$, and under naive implementation it holds that $w=3$. Thus our complexity guarantees are non-trivially stronger as compared with transformed graphs via MPNNs.

\subsection{Extensions}
\subsubsection{General dimension $d\neq 3$}

The construction developed in this manuscript is aimed at the practical scenario of machine learning where the point clouds reside in Euclidean space, i.e. $d=3$. Our proofs can naturally be extended to $d>3$. For instance, the universal representation of Euclidean equivariant polynomials, \cref{lem:orthovel}, is defined for $d=3$ but can be immediately applied to $d>3$ exactly as it is written. Furthermore, \cref{lem:w-wl-edge} can be generally formulated as embedding vectors of the form $(x_1, \ldots, x_{d-1}, \lms x_k \rms_{k=1}^{n})$, by considering higher order $k$-WL tests, specifically $(d-1)$-WL for $d$-dimensional point clouds. High order $k$-WL can be implemented by the high order variation of PPGN \cite{NEURIPS2019_provably}, and similar separation results hold, due to the density of separable functions.  

\subsubsection{Equivariance to proper rotations}
\cref{lem:orthovel} can naturally be applied on Proposition 5 \cite{villar} to obtain similar pooling operators that respect the Special Orthogonal group (proper rotations). Complete tests for $SO(3)\times S_n$ can be obtained by replacing the classical $2$-WL test with the similar geometric version ($2$-Geo) in \cite{hordan2023complete}.

\subsubsection{Multiple equivariant node features}

In some settings, such as using multiple position `channels' \cite{levy2023multiple}, we may wish to incorporate not only a pair of point clouds (positions and velocities) but a finite collection of such pairs. The results in this manuscript naturally extend to such settings, by defining edge weights $w_{ij}$ to be the $2k\times 2k$ distance matrix of the $k$ features per each point.

\subsection{Update Step of PPGN}\label{sub:update}

There are two approaches for implementing the `update' step in PPGN. \cite{NEURIPS2019_provably} originally offered to concatenate the $\bc_{(t)}(i,j) \in \RR^{D}$ coloring with $\tilde{\bc_{(t)}}(i,j) \in \RR^{D}$ to obtain $\bc_{(t+1)}(i,j) := (\bc_{(t)}(i,j), \tilde{\bc_{(t)}}(i,j)) \in \RR^{2D}$. This clearly maintains all the information, yet the dimensionality is exponentially dependent on the timestamp.

We use the approach implemented in \cite{li2023distance}, which is to multiply element-wise the two vectors, i.e. $\bc_{(t+1)}(i,j) := \bc_{(t)}(i,j) \odot \tilde{\bc_{(t)}}(i,j)) \in \RR^{D}$, which maintains a constant embedding dimension throughout the color refinement process.

\subsection{General Framework for Proving Equivariant Universality}\label{sec:genfram}

For simplicity, we'll focus on the case where we wish to approximate continuous equivariant functions on point clouds $X\in \RR^{3\times n}$, that is $f:\RR^{3\times n} \to \RR^{3\times n} $ which satisfies $f(RXP^{T} + t)=Rf(X)P^{T }+t$ for a representation of a permutation action $P$, a rotation (proper or improper) $R$ and a translation vector $t\in \RR^{3}$ (with vector addition element-wise).

$$f(X)_i = x_{i} + \sum_{k} \phi(x_i,x_j, \lms x_k\m k \neq i,j \rms)(x_k - x_i)$$

where $\phi$ is a continuous, euclidean invariant function and $x_i := X_i\in \RR^{3}$.

By \cite{dym2023low}, $\phi(x_i,x_j, \lms x_k\m k \neq i,j \rms)$ can be expressed as $\phi(x_i,x_j, \lms x_k\m k \neq i,j \rms) = \Tilde{\phi}\circ Embed(x_i,x_j,\lms x_k \;| \; k\neq i,j \rms)$ where $\Tilde{\phi}$ is continuous.

Any model that can produce an \textit{injective}, rotation and translation invariant $Embed$ function of the vectors, multi-set concatentation $(x_i,x_j, \lms x_k \; | \; k\neq i,j \rms)$ and approximates the the above aggregation is provably equivariant universal. We have shown 5 iterations of $2$-WL applied to point clouds can produce the required injective invariants in \cref{lem:w-wl-edge}.

\newpage
\section{Details on experiments}\label{app:exp}
\subsection{Implementation details}\label{subsec:imp-dets}

\WeLNet is an Euclidean equivariant neural-network-based GNN that has two main color refinement paradigms. One of them is $2$-WL equivalent and the other is $1$-WL equivalent (message passing). The interaction between these two paradigms yields the equivariant pooling operations \cref{eq:eqpoolx} and \cref{eq:eqpoolv}. We use PPGN\cite{NEURIPS2019_provably} to simulate the $2$-WL color refinement and base our message-passing-like color refinement on EGNN\cite{egnn}.

\WeLNet is defined as a successive application of convolution blocks, that involve a parameter-sharing scheme. We first initialize node-wise hidden states $\h:=\lms h_i\rms_{i=1}^{n}$ and edge features 
$\mathrm{E}:=(e_{ij})_{i,j=1}^{n}$ that contain node-wise features, such as atom number for molecules, and pair-wise features, such as magnetic attraction or repulsion for particles, respectively, as in EGNN. We also initialize a shared $2$-WL equivalent 
 analytic PPGN architecture that is used throughout all the convolution steps, denoted as $\PPGNan$.

Each convolution layer, denoted as \textbf{WeLConv}, takes as input a position point cloud, $X\in \Rthree$, and a velocity point cloud, $V\in \Rthree$, and outputs an updated position and velocity point clouds,  $(\h^{\mathrm{out}}, X^{\mathrm{out}}, V^{\mathrm{out}})=\welconv(\h, \mathrm{E}, \PPGNan, X,V)$.

We define the Convolution Layer of \WeLNet, i.e. $\welconv$, which is based on EGNN, as

\begin{align}
    \bc(i,j) &= \PPGNan(X,V)_{i,j}\\
    m_{ij} &= \phi_{e}(h_i, h_j, e_{ij}, \bc(i,j))\\
    m_i&= \sum_{j} m_{i,j}\\
    x_{i}^{\mathrm{out}} &= x_{i} + \phi_{n}(m_{i})v_i + \sum_{j} \phi_{x}(m_{i,j})(x_j - x_i) + \sum_{j\neq i}\phi_{v}(m_{i,j})v_j\\ \label{egnn-x-pool}
    v_i^{\mathrm{out}} &= \hat{\phi_{n}}(m_{i})v_i + \sum_{j} \hat{\phi_{x}}(m_{i,j})(x_j - x_i) + \sum_{j\neq i}\hat{\phi_{v}}(m_{i,j})v_j\\\label{egnn-v-pool}
    h_{i}^{out} &= \phi_h(h_i, m_i)
\end{align}

where the $\phi$'s are MLPs. 

Note that  Equations \ref{egnn-x-pool}-\ref{egnn-v-pool} are of the same form as the equivariant pooing layers defined in Equations \ref{eq:eqpoolx}-\ref{eq:eqpoolv}. Thus, our theory guarantees the universality of this construction with $T=5$ PPGN iterations and a single Convolution iteration, and dimensionality of $12n+1$ neurons for the size of the vector $\bc(i,j)$. In practice, we use $T=2$ PPGN iterations and $4$ Convolutions.

\subsection{N-Body Experiment}

The N-body problem is a physical dynamics problem that arises in a number of physical settings, such as the solar system, electrical charge configurations and double-spring pendulums, in which a model aims to predict the trajectory of objects that mutually assert forces on one another based on a physical law, e.g. gravity in the solar system. The contemporary standard benchmark is a dynamical system in $3\mathrm{D}$ space that models the time-dependent trajectory of $5$ particles with an electrical charge.
This task has been introduced in \cite{fuchs-body}, and \cite{egnn} extended the Charged
Particles N-body experiment from \cite{KipfFWWZ18} to a 3
dimensional space, which remains the standard setting for this task. Each particle carries a position coordinate, negative or positive charge, and an initial velocity. This system is equivariant to the symmetries described in this manuscript for position velocity point cloud pairs, where rotation and permutation act simultaneously on both position and velocity point clouds, while translation only acts on the position point cloud. This system respects these symmetries because the electromagnetic force between particles is equivariant to rotations and permutations.

\subsubsection{Dataset}

Following \cite{egnn}, we sampled 3,000 trajectories for training, 2,000
for validation and 2.000 for testing. Each trajectory has
a duration of 1.000 timesteps. For each trajectory, we
are provided with the initial particle positions
, their initial velocities
and their respective charges. The task is to estimate the positions of the five particles after 1.000 timesteps. We optimize
the averaged Mean Squared Error of the estimated position
with the ground truth one when training and test our performance using MSE, as well.

\subsubsection{Configuration}

We ran our experiment on an NVIDIA A40 GPU with CUDA toolkit version 12.1. Below is a table of the model configuration

\subsubsection{Results}

Results for the N-Body task are those reported by the papers. In the classical N-Body dataset (no external forces), we compare to other baselines including MC-EGNN \cite{ levy2023multiple}, CN-GNN \cite{DBLP:conf/icml/KabaMZBR23}, SEGNN \cite{brandstetter2021geometric}, FA-GNN \cite{puny2021frame}, ClofNet\cite{clofnet} and EGNN \cite{egnn}. 

In the external force experiment, We compare to ClofNet \cite{clofnet}, MC-EGNN \cite{levy2023multiple} and EGNN \cite{egnn}. Results for the custom Force task are reproduced.

\begin{table}[t]
\caption{Configuration of the \WeLNet Architecture.}
\label{config-table}
\vskip 0.15in
\begin{center}
\begin{sc}
\begin{tabular}{lcr}
\toprule
Hyperparameter & Value \\
\midrule
Activation    & Scaled Softplus  \\

Edge Features Dim & 128\\

WL Features Dim    & 32        \\
Learning rate  & 1e-3 \\

Optimizer     & Adam\\
Scheduler & StepLR\\
Number of Convolutions & 4\\
$2$-WL Iterations (T) & 2 \\ 

\bottomrule
\end{tabular}
\end{sc}
\end{center}
\vskip -0.1in
\end{table}
\subsubsection{Implementation}
For implementing $\PPGNan$ we first embedded the distances via exponential radial basis functions, relying on the implementation of these functions by \cite{li2023distance}. We rely on the implementation of PPGN by \cite{NEURIPS2019_provably} and incorporate modifications to this architecture introduced by \cite{li2023distance}. For the message-passing color refinement we rely on EGNN and MC-EGNN\cite{egnn,levy2023multiple} and incorporate multiple position channels as introduced in \cite{levy2023multiple}.

\subsection{Conformation Generation}

A conformation of a molecule is a $3\mathrm{D}$ such that the intra-molecular forces are in an equilibrium. Conformation generation aims to predict stable 3D conformations from 2D molecular graphs, which are a more natural representation of molecules \cite{DBLP:conf/icml/ShiLXT21}. We follow \cite{DBLP:conf/icml/ShiLXT21} and essentially estimate the gradients of the force fields to `move' the atoms towards an energy equilibrium. The key challenge with this approach is for a model to respect the equivariance of these gradients to rotations and translations and to accurately predict them. We leverage the geometric information embedded in the distance matrix of the positions and using \WeLNet we can equivariantly to estimate the direction of the gradient field along which the atoms move. This is a generative task in which we begin with a randomly sampled point cloud and then iteratively update its positions via the estimated gradient field to obtain a stable molecular conformation.

The metrics presented in \cref{tab:geom-qm9} are defined as

\begin{align}
    \mathrm{COV}(S_g,S_r) = \frac{1}{|S_r|} |\{ R\in S_r \m \mathrm{RMSD}(R, \hat{R})< \delta , \hat{R} \in S_g \}| \\
    \mathrm{MAT}(S_g, S_r) = \frac{1}{|S_r|} \sum_{R\in S_r} \underset{\hat{R}\in S_g}{\min} \mathrm{RMSD}(R,\hat{R})
\end{align}

for a given threshold $\delta$ where $S_g$ and $S_r$ denote generated and reference conformations, respectively, and the RMSD of heavy atoms  measures the distance between generated conformation and the reference. In the GEOM-QM9 \cite{geom},which  consists of small molecules up to 29 atoms, dataset this threshold is defined to be $\delta=0.5$. For further details please refer to \cite{DBLP:conf/icml/ShiLXT21}.

\subsubsection{Implementation Details}

We use the standard \WeLNet architecture described in septh in \cref{subsec:imp-dets}, but incorporate into the ConfGF \cite{DBLP:conf/icml/ShiLXT21} in similar fashion to that performed by ClofNet \cite{clofnet}. This means that we don't use EGNN as the message passing layer, but rather a Transformer-based GNN. For complete implementation details see \cite{DBLP:conf/icml/ShiLXT21}.

\subsubsection{Configuration}
The configuration is identical to that of the N-Body problem but we use different hyperparameters for the conformation generation task, as outlined in \cref{config-table-conf-gen}.

\subsubsection{Results}

Results of the other baselines are taken from \cite{zhou2023unimol}.  The compared models for molecular generation include RDKit\cite{rdkit}, CVGAE \cite{cvgae}, GraphDG \cite{GraphDG}, CGCF \cite{cgcf}, ConfVAE \cite{confvae}, ConfGF \cite{DBLP:conf/icml/ShiLXT21}, GeoMol \cite{geomol}, DGSM \cite{dgsm}, ClofNet \cite{clofnet}, GeoDiff \cite{geodiff}, and DMCG \cite{dmcg}. 
\begin{table}[t]
\caption{Configuration of the \WeLNet Architecture in the Conformation Generation task.}
\label{config-table-conf-gen}
\vskip 0.15in
\begin{center}
\begin{sc}
\begin{tabular}{lcr}
\toprule
Hyperparameter & Value \\
\midrule
Activation    & Scaled Softplus  \\

Edge Features Dim & 288\\

WL Features Dim    & 256        \\
Learning rate  & 3E-4 \\

Optimizer     & Adam\\
Scheduler & None\\
Number of Convolutions & 4\\
$2$-WL Iterations (T) & 3 \\ 
Generation Elapsed Time & 3.5 days\\
\bottomrule
\end{tabular}
\end{sc}
\end{center}
\vskip -0.1in
\end{table}

\subsection{Separation experiment}

In \cref{fig:sep}, we show the separation gap of 4 different activation function. Formally, the separation gap is the absolute value between the output of the first element of graph pair and the second. Our aim was to compare the separation power of two very popular activations, ReLU and LeakyReLU as copmared to, arguably, their analytic, non-polynomial approximations, Softplus and LeakyELU. LeakyELU is an activation not commonly used, which we have devised in order to approximate LeakyReLU. It is defined as 

\begin{equation}\label{eq:leaky-elu}
        \mathrm{LeakyELU}(x) = \mathrm{ELU}(x) - \alpha\cdot\mathrm{Softplus}(-x)
\end{equation}

for $\alpha > 0$.

where ELU and Softplus are popular activation functions defined as
\begin{align}
    \mathrm{ELU}(x) = & \begin{cases}
x, & x\geq 0\\
\exp(x)-1,  & x< 0
\end{cases}  \\
     \mathrm{Softplus}(x) = & \log(1 + \exp(x)) 
\end{align}

The experiment demonstrates that the separation power of analytic, non polynomial activations can be readily observed when comparing piecewise linear activations and their analytic, non-polynomial approximations. This falls in line with the separation results introduced in \cref{thm:uniform}, which for combinatorial graphs separation is guaranteed for any analytic, non-polynomial function only with one neuron width throughout the MLPs in the PPGN blocks.
\newpage
\section{Proofs}

\pairsep*

\begin{proof}[Proof of Theorem \ref{thm:pairs} ] 
We recall that the PPGN architecture as defined in the main text is initialized with input edge coloring $\bc_{(0)}(i,j)$ assigned from the input graph $\G$ (\cref{eq_PPGN_init}), and then iteratively applies $t=0,\ldots,T-1$ update steps as per \cref{eq:PPGNagg,eq_PPGN_update_2},
\begin{equation*}
\tilde{\bc}_{(t+1)}(\bi)=\sum_{j=1}^n \phi^{(1,t)}\left(\bc_{(t)}(i_1,j) \right)\odot \phi^{(2,t)}\left(\bc_{(t)}(j,i_2) \right)
\end{equation*}
and 
\begin{equation*}
\bc_{(t+1)}(\bi)=\tilde{\bc}_{(t+1)}(\bi) \odot \phi^{(3,t)}\left(\bc_{(t)}(\bi)\right).
\end{equation*}
 After $T$ iterations, a final graph-level representation $\cglobal$ is computed via
 $$\cglobal=\sum_{\bi\in[n]^2} \phi_{\text{READOUT}}(\bc_{(T)}(\bi)).$$

The various functions $\phi^{(s,t)}$ and $\phi_{\text{READOUT}}$ defined above are all shallow networks of the form $\sigma(a\cdot x+b) $, where $\sigma$ is an analytic non-polynomial function, and $a,b \in \RR$ (with  the possible exception of the first layer where the dimension of $a$ is the input feature dimension $D$. 

To prove pairwise separation, let $\G,\G'$ be two graphs that are separable by $T$ iterations of $2$-WL. We need to show that for Lebesgue almost every choice of the parameters of the MLPs $\phi^{(1,t)},\phi^{(2,t)},\phi^{(3,t)}$ and $\phi_{\text{READOUT}}$, the final features $\cglobal$ and $\cglobal'$ computed from $\G,\G'$ respectively are different. In fact, since $\cglobal$ and $\cglobal'$ are an analytic function of their parameters (for fixed inputs $\G,\G'$ respectively), it is sufficient to show that there \emph{exist} parameters such that $\cglobal \neq \cglobal'$. This is because an analytic function that is not identically zero is zero only on a set of Lebesgue measure zero (see Proposition 3 in \cite{mityagin2020}).

Due to the way PPGN simulates the $2$-WL process, we essentially need to show that, for every index pair $i,j$ and natural $t\leq T$, that if  
$\bc_{(t)}(i,j)\neq \bc_{(t)}'(i,j)$, or 
$$\lms \left( \bc_{(t)}(i,k),\bc_{(t)}(k,j) \right)\rms_{k=1,\ldots,n} \neq \lms \left( \bc_{(t)}'(i,k),\bc_{(t)}'(k,j) \right)\rms_{k=1,\ldots,n}  $$
then there exists a choice of the parameters $\theta_t$ of the $t$-th layer MLPs $\phi^{(s,t)}, s=1,2,3 $ such that
\begin{equation}\label{eq:Fsep1}
\bc_{(t+1)}(i,j)=F_{ij}(\bc_{(t)}(i,j),\theta_t)\neq F_{ij}(\bc_{(t)}'(i,j),\theta_t)= \bc_{(t+1)}'(i,j),\end{equation}
where we used $F_{ij}$ to denote the function creating the $t+1$ coloring of the $i,j$ entry from all previous colorings $\bc^{(t)} $, and $\theta_t$ to denote the parameters of these functions $F_{ij}$.

Additionally, we need to show that if after the $T$ PPGN iterations are concluded the finite feature multisets are distinct, that is 
$$\lms \bc_{(T)}(i,j) \rms_{1\leq i,j \leq n} \neq \lms \bc_{(T)}'(i,j) \rms_{1\leq i,j \leq n}  $$
then there exists a choice of the parameters of $\phi_{\text{READOUT}}$ such that
$$\sum_{i,j} \phi_{\text{READOUT}}(\bc_{(T)}(i,j)) \neq \sum_{i,j} \phi_{\text{READOUT}}(\bc_{(T)}'(i,j)).$$
This last part was already proven in \cite{amir2023neural}. Our goal is  to prove the first part.

Let us first assume that $\bc_{(t)}(i,j)\neq \bc_{(t)}'(i,j)$. Our goal is to show that there exists  $\theta_t$ such that Equation \eqref{eq:Fsep} holds. 

To show the existence of such parameters, we can choose the linear part of $\phi^{(1,t)}$ and $\phi^{(2,t)} $ to be zero and the bias to be non-zero so that we obtain $\tilde{\bc}_{(t+1)}(\bi)=\tilde{\bc}_{(t+1)}'(\bi) \neq 0 $. We can then choose the parameters of $\phi^{(3,t)}$ so that $\phi^{(3,t)}\left(\bc_{(t)}(\bi)\right)\neq \phi^{(3,t)}\left(\bc_{(t)}'(\bi)\right) $ which gives us what we wanted. Indeed, such parameters for $\phi^{(3,t)}$ can always be chosen: To show this we just need to show that if $y\neq y'$  are   $K$ dimesional  ($K=1$ for $t>0$ but not necessarily when $t=0$). we can choose a $K$ vector $a$ and bias $b$ such that 
$$ a\cdot y-b=0\neq a \cdot y'-b:=c.$$
Once applying a general analytic non-polynomial activation function $\sigma$ this equality may not be preserved. However, there will be a scaling  $s\in \RR$ of these parmaters such that 
$$\sigma( sa\cdot y'-sb)-\sigma(sa\cdot y-sb)=\sigma(sc)-\sigma(0)\neq 0, $$
because $\sigma$ is non-polynomial and in particular non-constant. 

We now consider the more challenging case where
$\bc_{(t)}(i,j)= \bc_{(t)}'(i,j)$, but
$$\lms \left( \bc_{(t)}(i,k),\bc_{(t)}(k,j) \right)\rms_{k=1,\ldots,n} \neq \lms \left( \bc_{(t)}'(i,k),\bc_{(t)}'(k,j) \right)\rms_{k=1,\ldots,n}  $$
we choose the parameters of $\phi^{(3,t)}$ so that it is a constant non-zero function. We need to show that we can choose the parameters of $\phi^{(1,t)}$ and  $\phi^{(2,t)}$ so that $\tilde{\bc}_{(t+1)}(i,j)\neq  \tilde{\bc}_{(t+1)}'(i,j)$. 

For simplicity of notation, we introduce the notation
$$x_k^{(1)}=\bc_{(t)}(i,k)   \text{ and } x_k^{(2)}=\bc_{(t)}(j,k).$$
and  
$$\bar{x}_k^{(1)}=\bc_{(t)}'(i,k) \text{ and } \bar{x}_k^{(2)}=\bc_{(t)}'(j,k).$$

Our goal is to prove the following lemma
\begin{lemma} \label{lem:main_sep} If $(x_k^{(1)},x_k^{(2)}) $ in $\RR^{d_1}\oplus\RR^{d_2}$ and  $(\bar{x}_k^{(1)},\bar{x}_k^{(2)}) $ in $\RR^{d_1}\oplus\RR^{d_2}$ such that 
\begin{equation}\label{eq:ms_ineq}
\lms (x^{(1)}_k, x^{(2)}_k ) \rms_{k=1}^n \neq \lms (\bar x^{(1)}_k, \bar x^{(2)}_k )\rms_{k=1}^n   
\end{equation}
Let $\sigma:\RR \to \RR$ be a continuous non-polynomial function. Then there exists a choice of $a^{(1)}\in \RR^{d_1},a^{(2)}\in \RR^{d_2} $ and $b^{(1)},b^{(2)} \in \RR $ such that 
\begin{align*}\sum_k \sigma(a^{(1)}\cdot x_k^{(1)}+b^{(1)} ) \sigma(a^{(2)}\cdot x_k^{(2)}+b^{(2)} )\\ 
\neq \sum_k \sigma(a^{(1)}\cdot \bar{x}_k^{(1)}+b^{(1)} ) \sigma(a^{(2)}\cdot \bar{x}_k^{(2)}+b^{(2)} )  
\end{align*}
\end{lemma}
We prove this claim in a number of steps. First, we use $\mathcal{K}$ to denote the collection of all pairs  $(x_k^{(1)},x_k^{(2)})$ and $(\bar{x}_k^{(1)},\bar{x}_k^{(2)})$. Note that this set is finite and in particular compact. 

Due to the multiset inequality \eqref{eq:ms_ineq}, there exists some fixed entry $\kappa$ such that 
$(x_{\kappa}^{(1)},x_{\kappa}^{(2)})$ does not appear in the first multiset the same amount of times as it appear in the second multiset. Let $f$ be a continuous function on $\RR^{d_1+d_2}$ satisfying that 
$f(x_{\kappa}^{(1)},x_{\kappa}^{(2)})=1$ and $ f(x^{(1)},x^{(2)})=0$  for all other  $(x^{(1)},x^{(2)}) \in \mathcal{K}$. Then 
\begin{equation}\label{eq:fsep}\sum_k f(x_{k}^{(1)},x_{k}^{(2)})\neq \sum_k f(\bar x_{k}^{(1)},\bar x_{k}^{(2)})  \end{equation}
Next, we wish to show that the same separation can be obtained by a finite linear combination of continuous separable functions, that is, linear combinations of functions of the form $f(x^{(1)},x^{(2)})=f_1(x^{(1)})\cdot f_2(x^{(2)}) $ where $f_1:\RR^{d_1}\to \RR$ and $f_2:\RR^{2_1}\to \RR$ are continuous. To do this we use the well-known Stone-Weierstrass theorem 

\begin{theorem}[Stone–Weierstrass theorem ( compact spaces)]
Suppose $\mathcal{K}$ is a  compact Hausdorff space and $A$ is a subalgebra of $C(\mathcal{K}, \mathbb{R})$. Then $A$ is dense in $C(\mathcal{K}, \mathbb{R})$ in the topology of uniform convergence if and only if it separates points and containts a non-zero constant function. 
\end{theorem}
Let $A$ be the set of finite sums of continuous separable functions. We know that continuous functions separate points by Stone-Weierstrass, and thus it can easily be shown that separable functions do as well (for $x\neq y$ there exists some coordinate $j$ such that $x_j \neq y_j$ and then let one of the separated functions of $f$ satisfies $g(x^{1})\neq g(y^{1})$ such that $x_j \in x^{(1)}, y_j \in y^{(1)}$ w.l.o.g and the other separated functions be constant $h\equiv 1$, then we have that $f(x)\neq f(y)$ as desired) and it's easy to see that non-zero constant functions exists in $A$  well. Thus the algebra generated by separable functions is dense in continuous functions by Stone-Weierstrass. In particular we can obtain \eqref{eq:fsep} for $f$ which is a finite linear combination of separable functions.

Next, due to the universality of shallow neural networks with non-polynomial activations \cite{pinkus1999approximation}, we can approximate each separable function  $f_1(x^{(1)})\cdot f_2(x^{(2)} )$ to arbitrary accuracy be expressions of the form 
$$\sum_{s=1}^S \sigma(a^{(1,s)}\cdot x^{(1)}+b^{(1,s)}) \sum_{s=1}^S \sigma(a^{(2,s)}\cdot x^{(2)}+b^{(2,s)}).$$

By changing order of summation and multiplication,
We deduce that for appropriate choice of $$a^{(1,s)},a^{(1,s)},b^{(1,s)},b^{(2,s)}, s=1,\ldots,S $$
we have that 
\begin{align*}
\sum_{s=1}^S \sum_k  \sigma(a^{(1,s)}\cdot x_k^{(1)}+b^{(1,s)} ) \sigma(a^{(2,s)}\cdot x_k^{(2)}+b^{(2,s)} )\\ 
\neq \sum_{s=1}^S \sum_k \sigma(a^{(1,s)}\cdot \bar{x}_k^{(1)}+b^{(1,s)} ) \sigma(a^{(2,s)}\cdot \bar{x}_k^{(2)}+b^{(2,s)} ) .
\end{align*}
In particular there must be an inequality for at least one $s$ which concludes the lemma, and the proof of pairwise separation (Theorem \ref{thm:pairs}). 
\end{proof}
\uniformwl*
\begin{proof}[Proof of Theorem \ref{thm:uniform}]
To obtain uniform separation (Theorem \ref{thm:uniform}) we need to show that, for a given $\sigma$-subanalytic set $\X \subseteq \RR^{n \times n \times k}$ of dimension $D$,
if the MLPs $\bphi$ in PPGN are all taken to be of the form $x\mapsto \sigma(Ax+b) $,
where $\sigma$ is an analytic non-polynomial activation function and $Ax+b$ is a vector in $\RR^{2D+1}$,
then for Lebesgue almost every choice of the network's parameters $\theta$, 
a pair $\G,\G' \in \X$ can be separated by $T$ iterations of 2-WL if and only if the global features $\cglobal$ and $\cglobal'$ are distinct, where $\cglobal$ and $\cglobal'$ are obtained by applying the PPGN network with the parameters $\theta$ to $\G$ and $\G'$ respectively

Equivalently, we need to show recursively on $t$, that for Lebesgue almost every choice of the $t$-th layer, we will have for every index pair $i,j$ and every $\bc_{(t)}, \bc_{(t)}'$ which were obtained by applying the first $t-1$ layers to the initial coloring $\bc_{(0)}$ and  $\bc_{(0)}'$ induced from $\G,\G'$ with parameters which were already chosen, if  
$\bc_{(t)}(i,j)\neq \bc_{(t)}'(i,j)$, or 
$$\lms \left( \bc_{(t)}(i,k),\bc_{(t)}(k,j) \right)\rms_{k=1,\ldots,n} \neq \lms \left( \bc_{(t)}'(i,k),\bc_{(t)}'(k,j) \right)\rms_{k=1,\ldots,n}  $$
Then for Lebesgue almost every choice of the parameters $\hat \theta_t$ of the $t$-th layer MLPs $\phi^{(s,t)}, s=1,2,3 $ we will get that 
\begin{equation}\label{eq:Fsep}
\bc_{(t+1)}(i,j)=\hat{F}_{ij}(\bc_{(t)}(i,j),\hat{\theta}_t)\neq \hat{F}_{ij}(\bc_{(t)}'(i,j),\hat \theta_t)= \bc_{(t+1)}'(i,j),\end{equation}
where $\hat{F}_{ij}$ denotes the function mapping $\bc_{(t)}$ to the $i,j$ index of the $t$-th layer. We note that the function $\hat{F}_{ij}$ whose output is $2D+1$ dimensional consists of $2D+1$ copies of the one dimensional PPGN function $F_{ij}$ from the previous theorem, that is 
$$\hat{F}_{ij}(\bc_{(t)}; \hat \theta_t)=\left(F_{ij}(\bc_{(t)};\theta_t^{(1)}),\ldots,F_{ij}(\bc_{(t)};\theta_t^{(2D+1)})\right). $$

We will now use the finite witness theorem, which essentially allows for moving from pairwise separation to uniform separation by taking enough clones of the pairwise separating functions
\begin{theorem}
    [Special case of Corollary A20 in \cite{amir2023neural}]
Let $\M \subseteq \RR^p$ be a $\sigma$-subanalytic sets of dimension $D$. Let $F:\M \times \RR^q \to \RR$ be a $\sigma$-subanalytic function which  is analytic as a function of $\bth$ for all fixed $\bx \in \M$. Define the set
$$\N=\{\left(\bx,\by\right) \in \M \times \M \setmid F(\bx;\bth)=F(\by;\bth),\ \forall \bth\in \RR^q \}.$$
Then for generic $\left( \bth^{(1)},\ldots,\bth^{(2D+1)}\right) \in \W^{D+1}$,
\begin{equation}
\N=\{\left(\bx,\by\right) \in \M \times \M \setmid F(\bx;\bth^{(i)})=F(\by;\bth^{(i)}),\ \forall i=1,\ldots 2D+1\}.
\end{equation}
\end{theorem}
We apply this theorem to the set $\mathcal{M}=\X_{(t)}$ which is the ouput of the first $t$ layers of the network applied to graphs in $\X$, and for the function $F=F_{ij}$. The set $\M$ is the image of a $\sigma$-subanalytic set under an analytic function, and so according to \cite{amir2023neural} this is a $\sigma$-subanalytic set of dimension $\leq D $. The set 
$$\N=\{(\bc_{(t)}, \bc_{(t)}') \in \X_{(t)}\times \X_{(t)} \setmid F_{ij}(\bc_{(t)};\theta_t)=F_{ij}(\bc_{(t)}';\theta_t) \forall \theta_t  \} $$
is precisely the set of $\bc_{(t)},\bc_{(t)}'$ which won't be assigned a different $(i,j)$ labeling at timestamp $t+1$ by 2-WL since both $\bc_{(t)}(i,j)= \bc_{(t)}'(i,j)$, and
$$\lms \left( \bc_{(t)}(i,k),\bc_{(t)}(k,j) \right)\rms_{k=1,\ldots,n} = \lms \left( \bc_{(t)}'(i,k),\bc_{(t)}'(k,j) \right)\rms_{k=1,\ldots,n}$$  

Thus, for almost every choice of $\theta_t^{(1)},\ldots, \theta_t^{(2D+1)} $, every pair of $\bc_{(t)},\bc_{(t)}'$ which \emph{will} be assigned a different $(i,j)$ labeling at timestamp $t+1$ by 2-WL, will be assigned a different labeling by $\hat F $ as well. This concludes the proof of the theorem. 
\end{proof}

\newpage

\begin{restatable}{lemma}{orthovel}\label{lem:orthovel}
	Let $\Psi : \br{\RR^{3 \times n}}^2 \to \br{\RR^{3 \times n}}^2$ be a $O\of 3$- and $S_n$-equivariant function, $\Psi\of{X,V} = \left[\Psi_1\of{X,V},\Psi_2\of{X,V}\right]$, where $X,V,\Psi_1\of{X,V},\Psi_2\of{X,V} \in \RR^{3 \times n}$. Denote for $c=1,2$, $\Psi_c\of{X,V} = \br{\Psi_c^1\of{X,V},\ldots,\Psi^n_c\of{X,V}}$. Then $\Psi$ can be expressed as 
    \begin{equation}\label{eq:pool-rot-only}
    \begin{split}
        \Psi_c^i\of{X,V}
         = &f_c((x_i,v_i), \lms \br{x_k,v_k} \m k \neq i \rms ) x_i
         \\ +& \hat f_c((x_i,v_i), \lms \br{x_k,v_k} \m k \neq i \rms ) v_i
        \\ +& \sum_{t \in \left[n\right]\setminus i} g_c((x_i, v_i),(x_t, v_t),  \lms (x_k, v_k) \m k\neq i,t\rms)x_t
        \\ +& \sum_{t \in \left[n\right]\setminus i} \hat{g}_c((x_i, v_i),(x_t, v_t),  \lms (x_k, v_k) \m k\neq i,t\rms)v_t
    \end{split}
    \end{equation}
 for $c=1,2$, $i \in \left[n\right]$, where the functions $f_c,\hat f_c, g_c,\hat g_c : \br{\RR^{3 \times n}}^2 \to \RR$ for $c=1,2$ are $O\of 3$-invariant. Moreover, if $\Psi$ is a polynomial, then the $f_c,\hat f_c, g_c,\hat g_c$ can be taken to be polynomials.
\end{restatable}
\begin{proof}
    It is enough to prove the lemma for $c=1$, as the same argument can be repeated for $c=2$. We thus omit $c$ from the notation and treat $\Psi$ as a function from $\br{\RR^{3 \times n}}^2$ to $\RR^{3 \times n}$. 

    We start by stating the following proposition.
    \begin{proposition}\label{prop_tal_proof_1}
        Let $h: \br{\RR^{3 \times n}}^2 \to \RR^3$ be $O\left(3\right)$-equivariant.
        Then there exist $O(3)$-invariant functions $f_i, \hat f_i : \br{\RR^{3 \times n}}^2 \to \RR$, $i=1,\ldots,n$, such that
        \begin{equation}
            h\of{X,V} = \sum_{i=1}^n [f_i\of{X,V}x_i + \hat f_i\of{X,V}v_i].
        \end{equation}
        Moreover, if $h$ is polynomial, then the $f_i$'s and $\hat f_i$'s can also be taken to be polynomials.
    \end{proposition}

    \begin{proof}
        \Cref{prop_tal_proof_1} follows from treating $h$ as a function from $\RR^{3 \times 2n}$ to $\RR^3$ and applying Proposition 4 of \cite{villar}.
    \end{proof}
    
    We shall now prove the lemma. Since each $\Psi^i$ is $O\of 3$-equivariant, by \cref{prop_tal_proof_1} it can be presented as
    \begin{equation*}
        \Psi^i\of{X,V} = \sum_{t=1}^n f^i_t\of{X,V}x_t + \hat f^i_t\of{X,V}v_t,
    \end{equation*}
    and thus for each $\sigma \in S_n$,
    \begin{equation}\label{tal_proof_expression_1}
        \Psi^i\of{\sigma X, \sigma V} = \sum_{t=1}^n f^i_t\left(\sigma X,\sigma V\right)x_{\sigma^{-1}\of t} + \hat f^i_t\of{\sigma X, \sigma V}v_{\sigma^{-1}\of t}.
    \end{equation}    
    Since $\Psi$ is $S_n$-equivariant,
    \begin{equation*}
        \Psi\of{\sigma X, \sigma V} = \sigma \Psi\left(X, V\right),
    \end{equation*}
    which, expanded by $i$, reads as
    \begin{equation}\label{tal_proof_expression_2}
        \Psi^i\of{\sigma X, \sigma V} = \Psi^{\sigma^{-1}\of{i}}\left(X, V\right).
    \end{equation}
        
    Combining \cref{tal_proof_expression_1,tal_proof_expression_2}, we get
    \begin{equation*}
    \begin{split}
        \Psi^{\sigma^{-1}\of i}\left(X, V\right) 
        = &\sum_{t=1}^n f^i_t\of{\sigma X, \sigma V}x_{\sigma^{-1}\left(t\right)} + \hat f^i_t\of{\sigma X, \sigma V}v_{\sigma^{-1}\left(t\right)}
        \\ = &\sum_{t=1}^n f^i_{\sigma\left(t\right)}\of{\sigma X, \sigma V}x_t + \hat f^i_{\sigma\left(t\right)}\of{\sigma X, \sigma V}v_t,
    \end{split}
    \end{equation*}    
    with the last equality resulting from replacing $t$ by $\sigma \of t$.
    Replacing $i$ by $\sigma \of i$ yields
    \begin{equation}\label{tal_proof_expression_3}
    \begin{split}
        \Psi^i\left(X, V\right) 
        = \sum_{t=1}^n f^{\sigma\of{i}}_{\sigma\left(t\right)}\of{\sigma X, \sigma V}x_t + \hat f^{\sigma\of{i}}_{\sigma\left(t\right)}\of{\sigma X, \sigma V}v_t,
    \end{split}
    \end{equation}
    with \cref{tal_proof_expression_3} holding for any $\sigma \in S_n$. Averaging over $S_n$ yields
    \begin{equation}\label{tal_proof_expression_4}
    \begin{split}
        \Psi^i\left(X, V\right) 
        = \sum_{t=1}^n \br{ \frac{1}{|S_n|}\sum_{\sigma \in S_n } f^{\sigma\of{i}}_{\sigma\left(t\right)}\of{\sigma X, \sigma V} }x_t + \br{ \frac{1}{|S_n|}\sum_{\sigma \in S_n } \hat f^{\sigma\of{i}}_{\sigma\left(t\right)}\of{\sigma X, \sigma V} } v_t.
    \end{split}
    \end{equation}
    Define $\tilde f^{i}_t, \tilde{\hat f}^{i}_t : \br{\RR^{3 \times n}}^2 \to \RR$ by
    \begin{equation*}
        \tilde f^{i}_t \of{X,V} = \frac{1}{|S_n|}\sum_{\sigma \in S_n } f^{\sigma\of{i}}_{\sigma\left(t\right)}\of{\sigma X, \sigma V},
        \quad
        \tilde{\hat f}^{i}_t \of{X,V} = \frac{1}{|S_n|}\sum_{\sigma \in S_n } \hat f^{\sigma\of{i}}_{\sigma\left(t\right)}\of{\sigma X, \sigma V},        
    \end{equation*}
    then \cref{tal_proof_expression_4} can be reformulated as
    \begin{equation}\label{tal_proof_expression_B1}
    \begin{split}
        \Psi^i\left(X, V\right) 
        = \sum_{t=1}^n \tilde f^{i}_t \of{X,V} x_t +  \tilde{\hat f}^{i}_t \of{X,V} v_t.
    \end{split}
    \end{equation}
    Let $\tau \in S_n$. Then
    \begin{equation}\label{tal_proof_expression_5}
    \begin{split}
        \tilde f^{i}_t \of{\tau X, \tau V} 
        = &\frac{1}{|S_n|}\sum_{\sigma \in S_n } f^{\sigma\of{i}}_{\sigma\left(t\right)}\of{\sigma \tau X, \sigma \tau V}
        \\ =& \frac{1}{|S_n|}\sum_{\sigma \in S_n } f^{\sigma \tau \of{\tau^{-1} i}}_{\sigma \tau \of{\tau^{-1} t}}\of{\sigma \tau X, \sigma \tau V}
        \\ =& \frac{1}{|S_n|}\sum_{\sigma \m \sigma \tau \in S_n } f^{\sigma \tau \of{\tau^{-1} i}}_{\sigma \tau \of{\tau^{-1} t}}\of{\sigma \tau X, \sigma \tau V}
        \\ \overset{(a)}{=}& \frac{1}{|S_n|}\sum_{\sigma \in S_n } f^{\sigma \of{\tau^{-1} i}}_{\sigma \of{\tau^{-1} t}}\of{\sigma X, \sigma V}
        \\ \overset{(b)}{=}& \tilde f^{\tau^{-1} \of i}_{\tau^{-1} \of t} \of{X, V},
    \end{split}
    \end{equation}
    with (a) following from replacing $\sigma \tau$ by $\sigma$, since both permutations iterate over all of $S_n$, and (b) holding by the definition of $f^{i}_t \of{X,V}$. 
    By replacing $i$ by $\tau \of i$ and $t$ by $\tau \of t$ in \cref{tal_proof_expression_5}, we get that for any $\tau \in S_n$,
    \begin{equation}\label{tal_proof_expression_6a}
        \tilde f^{\tau \of i}_{\tau \of t} \of{\tau X, \tau V}
        = 
        \tilde f^{i}_{t} \of{X, V},
    \end{equation}
    and applying the same reasoning to $\tilde{\hat f}^{i}_t$, we get 
    \begin{equation}\label{tal_proof_expression_6b}
        \tilde {\hat f}^{\tau \of i}_{\tau \of t} \of{\tau X, \tau V}
        = 
        \tilde {\hat f}^{i}_{t} \of{X, V}.
    \end{equation}

    Finally, define $f : \RR^{3 \times 2} \times \RR^{3 \times \br{n-1}} \times \RR^{3 \times \br{n-1}} \to \RR$ by
    \begin{equation}\label{tal_proof_expression_7a}
        f\of{\br{x,v}, \tilde X, \tilde V}
        = \tilde f^1_1 \of {\left[x,\tilde X\right], \left[ v, \tilde V\right]}
    \end{equation}
    and define $g : \RR^{3 \times 2} \times \RR^{3 \times 2} \times \RR^{3 \times \br{n-2}} \times \RR^{3 \times \br{n-2}} \to \RR$ by
    \begin{equation}\label{tal_proof_expression_7b}
        g\of{\br{x_1,v_1}, \br{x_2, v_2}, \tilde X, \tilde V}
        = \tilde f^1_2 \of {\left[x_1,x_2,\tilde X\right], \left[ v_1, v_2, \tilde V\right]}.
    \end{equation}
    We first show that $f$ and $g$ are respectively $S_{n-1}$- and $S_{n-2}$-invariant. Let $\tilde\tau \in S_{n-1}$.
    \begin{equation*}
        f\of{\br{x,v}, \tilde\tau \tilde X, \tau \tilde V}
        =
        \tilde f^1_1 \of {\left[x,\tilde\tau \tilde X\right], \left[ v, \tilde\tau \tilde V\right]}.
    \end{equation*}
    We can augment $\tilde\tau \in S_{n-1}$ to a permutation $\tau \in S_n$ that fixes $1$, and have
    \begin{equation*}
        \tilde f^1_1 \of {\left[x,\tilde\tau \tilde X\right], \left[ v, \tilde\tau \tilde V\right]}
        =
        \tilde f^1_1 \of {\tau \left[x, \tilde X\right], \tau \left[ v, \tilde V\right]}.
    \end{equation*}
    If we show that $\tilde f^1_1$ is invariant to permutations that fix 1, then we have
    \begin{equation*}
        \tilde f^1_1 \of {\tau \left[x, \tilde X\right], \tau \left[ v, \tilde V\right]}
        =
        \tilde f^1_1 \of {\left[x, \tilde X\right], \left[ v, \tilde V\right]}
        =
        f\of{\br{x,v}, \tilde X, \tau \tilde V},
    \end{equation*}
    which would imply
    \begin{equation*}
        f\of{\br{x,v}, \tilde\tau \tilde X, \tau \tilde V}
        =
        f\of{\br{x,v}, \tilde X, \tau \tilde V}.
    \end{equation*}
    Indeed, if $\tau \in S_n$ is some permutation that fixes 1, then by \cref{tal_proof_expression_6a},
    \begin{equation*}
        \tilde f^1_1\of{\tau X, \tau V}
        = 
        \tilde f^{\tau\of 1}_{\tau\of 1}\of{\tau X, \tau V}
        =
        \tilde f^{1}_{1}\of{X, V}.
    \end{equation*}
    Hence, $f$ is $S_{n-1}$-invariant. A similar argument can be used on $g$, by letting $\tau \in S_n$ be an arbitrary permutation that fixes $1$ and $2$. Then
    \begin{equation*}
        \tilde f^1_2\of{\tau X, \tau V}
        = 
        \tilde f^{\tau\of 1}_{\tau\of 2}\of{\tau X, \tau V}
        =
        \tilde f^{1}_{1}\of{X, V},
    \end{equation*}
    which, combined with \cref{tal_proof_expression_7b}, implies that $g$ is $S_{n-2}$-invariant.
    
    We shall now show that 
    \begin{equation}\label{tal_proof_expression_8}
        \tilde f^i_t \of{X,V} = 
        \begin{cases}
            f\of{x_i,v_i, X_{\setminus i}, V_{\setminus i}} & t=i,
            \\ g\of{\br{x_i,v_i}, \br{x_t,v_t}, X_{\setminus i,t}, V_{\setminus i,t,}} & t \neq i.
        \end{cases} 
    \end{equation}
    Suppose first that $t=i$. Let $\tau = \br{i,1} \in S_n$. Then
    \begin{equation*}
    \begin{split}
        \tilde f^{i}_i \of{X, V}
        \overset{(a)}{=} &
        f^{\tau \of i}_{\tau \of i} \of{\tau X, \tau V}
        \\ = &
        f^{1}_{1} \of{\tau X, \tau V}
        \\ \overset{(b)}{=} &
        f \of{x_i,v_i, X_{\setminus i}, V_{\setminus t}},
    \end{split}
    \end{equation*}
    with (a) following from \cref{tal_proof_expression_6a}, and (b) following from the definition of $f$ in \cref{tal_proof_expression_7a}.
    Now suppose that $t \neq i$. Let $\tau \in S_n$ be the composition of 2-cycles defined by
    %
    \begin{equation*}
        \tau =
        \begin{cases}
            \br{1,2} & i=2 \text{ and } t=1
            \\ \br{t,2}\br{i,1} & i=2 \text{ and } t \neq 1
            \\ \br{i,1}\br{t,2} & \text{otherwise}.
        \end{cases}
    \end{equation*}
    In all cases above, $\tau\of i = 1$ and $\tau\of t = 2$. Thus, by \cref{tal_proof_expression_6a},
    \begin{equation*}
    \begin{split}
        \tilde f^{i}_t \of{X, V}
        & =
        \tilde f^{\tau \of i}_{\tau \of t} \of{\tau X, \tau V}
        \\ & =
        \tilde f^{1}_{2} \of{\tau X, \tau V}
        \\ & =
        \tilde f^{1}_{2} \of{\left[x_i,x_t,X_{\setminus i,t}\right], \left[v_i,v_t,V_{\setminus i,t}\right]}
        \\ & =
        g \of{\br{x_i,v_i},\br{x_t,v_t}, X_{\setminus i,t}, V_{\setminus i,t}}.
    \end{split}
    \end{equation*}
    Hence, \cref{tal_proof_expression_8} holds.

    Using the same procedure as above, one can construct functions
    \begin{equation*}
    \begin{split}
        \hat f &: \RR^{3 \times 2} \times \RR^{3 \times \br{n-1}} \times \RR^{3 \times \br{n-1}} \to \RR    
        \\
        \hat g &: \RR^{3 \times 2} \times \RR^{3 \times 2} \times \RR^{3 \times \br{n-2}} \times \RR^{3 \times \br{n-2}} \to \RR
    \end{split}
    \end{equation*}
    that are $S_{n-1}$- and $S_{n-2}$-invariant respectively, $O\of 3$ equivariant, and
    \begin{equation}\label{tal_proof_expression_B2}
        \tilde {\hat f}^i_t \of{X,V} = 
        \begin{cases}
            \hat f\of{x_i,v_i, X_{\setminus i}, V_{\setminus i}} & t=i,
            \\ \hat g\of{\br{x_i,v_i}, \br{x_t,v_t}, X_{\setminus i,t}, V_{\setminus i,t,}} & t \neq i,
        \end{cases} 
    \end{equation}
    In conclusion, \cref{tal_proof_expression_B1,tal_proof_expression_8,tal_proof_expression_B2}, combined with the fact that that $f,\hat f$ are $S_{n-1}$-invariant and $g,\hat g$ are $S_{n-2}$-invariant, and all these functions are $O\of 3$-equivariant, imply that $f, \hat f, g, \hat g$ satisfy \cref{eq:pool-rot-only}.
    
    Lastly, note that by \cref{prop_tal_proof_1}, if $\Psi$ is a polynomial, then $f, \hat f, g, \hat g$ can be taken to be averages of polynomials, and thus they are in turn also polynomials.
\end{proof}

\newpage
\begin{restatable}{lemma}{twowledge}\label{lem:w-wl-edge}
	Let $X\in \RR^{3 \times n}$ be a point cloud. Denote by $ \ctfive $ the pairwise coloring induced by 5 iterations of 2-WL update steps applied to the distance matrix induced by $X$. Then $\ctfive$ constitutes a translation, rotation and reflection invariant embedding of $(x_i, x_j, \lms x_k  \m k\neq i,j  \rms)$. If $x_i, x_j$ are not both degenerate, i.e. do not both equal the barycenter, then  $\ctfour$ is sufficient.
\end{restatable}
\begin{proof}
    We first highlight the difference between standard point cloud recovery and the above result. 3 iterations of $2$-WL are sufficient to recover the original 3-dimensional point cloud up to orthogonal and permutation actions \cite{rose2023iterations}. In this theorem, we wish to recover a \textit{labeled} point cloud, a more difficult problem, where we distinguish a particular pair of points and recover the rest of the points as a set. Although many of the lemmas in \cite{rose2023iterations} are used we solve a qualitatively different optimization problem. 

    We begin with the case at either $x_i$ or $x_j$ does not equal the barycenter (center) of the point cloud, which is defined as $\frac{1}{n}\sum_{k=1}^{n} x_k$.

We first introduce some preliminary definitions introduced by \citeauthor{rose2023iterations}:

Let $\mathbf{x} = (x_1,...,x_k) \in \mathbb{R}^d$ be a $k$-tuple of points in $\mathbb{R}^d$. The distance matrix of $x$ is the $k \times k$ matrix $A$ given by

$$A_{ij} = d(x_i, x_j), \quad i,j=1,....,k$$

Now, let $S \subset \mathbb{R}^d$ be a finite set. Then the \textbf{distance profile} of $\mathbf{x}$ w.r.t. $S$ is the multiset
\begin{equation}\label{def:distprofile}
   D_{\mathbf{x}} = \lms (d(x_1,y), d(x_2,y), \ldots, d(x_k, y)  \mid y \in S\rms 
\end{equation}
 


Let
\[b=\frac{1}{|S|}\sum_{y\in S}y\]
denote the barycenter of $S$. For a finite set $G \subset \mathbb{R}^d$, we denote by
$\text{LinearSpan}(G)$ the linear space spanned by $G$, and by $\text{AffineSpan}(G)$ the corresponding affine one. Their respective dimensions will be denoted by $\text{LinearDim}(G)$ and $\text{AffineDim}(G)$.

\begin{definition}
Let $S \subset \mathbb{R}^d$ be a finite set and let $b$ be its barycenter. A $d$-tuple $x = (x_1, \dots, x_d) \in \mathbb{S}^d$ satisfies the cone condition if
\begin{itemize}
\item $\text{AffineDim} (b, x_1, \dots, x_d) = \text{Affine Dim}(S)$,
\item if $\text{Affine Dim}(S) = d$, then there is no $x \in S$ such that $x-b$ belongs to the interior of $\text{Cone}(x_1 - b, \dots, x_d - b)$.
\end{itemize}
\end{definition}

\begin{definition}\label{def:enhancedprof}
For a tuple $x = (x_1, \dots, x_d) \in \mathbb{S}^d$, we define its enhanced profile as
\[EP(x_1, \dots, x_d) = (A, M_1, \dots, M_d),\]
where $A$ is the distance matrix of the tuple $(b, x_1, \dots, x_d)$ and $M_i = D_x[b/i]$ is the distance profile (see Equation \ref{def:distprofile}) of the tuple $(x_1, \dots, x_{i - 1}, b, x_{i + 1}, \dots, x_d)$ with respect to $S$.
\end{definition}

\begin{definition}
Let $S \in \mathbb{R}^d$ be a finite set and let $b$ be its barycenter. An \textbf{initialization data} for $S$ is a tuple $(A, M_1, \dots, M_d)$ such that $(A, M_1, \dots, M_d) = EP(x_1, \dots, x_d)$ for some $d$-ple $x = (x_1, \dots, x_d) \in \mathbb{S}^d$ satisfying the cone condition.
\end{definition}

We now introduce two main lemmas from \cite{rose2023iterations} and a corollary derived from them that we will make use of in out proof:

\paragraph{Lemma 3.7}\label{lem:3.7}{\cite{rose2023iterations}} For any tuple $(x_i,x_j)\in S^2$, from its coloring after one iteration of $2$-WL, $\ctone$, and the multiset $\lms \ctoneni(k,l) \m k,l \in [n]\rms $, we can recover the tuple of distances $(d(b,x_i),d(b,x_j))$.

\paragraph{Lemma 3.8}\label{lem:3.8}{\cite{rose2023iterations}} For any tuple $(x_i,x_j)\in S^2$, from its coloring after two iterations of $2$-WL, $\cttwo$, and the multiset $\lms \ctoneni(k,l) \m k,l \in [n]\rms $, we can recover the distance profile of the tuple $(b,x_i,x_j)$. 

\paragraph{Corollary.} For any triplet $\mathbf{x}=(x_i, x_j, x_k) \in S^3$ we can recover its enhanced profile, see \cref{def:enhancedprof}, from the tuple $(\cttwoni(j,k), \cttwoni(i,k), \cttwo)$ and the multiset $\lms \ctoneni(k,l) \m k, l \in [n]\rms$.
\begin{proof}[Proof of Corollary.]
    From Lemma 3.7, we can recover the distance matrix of $(b, x_i, x_j, x_k)$ which is the $A$ in the definition of the Enhanced Profile, see \cref{def:enhancedprof}. W.l.o.g  $D_{[b\setminus 3]}$ for can be recovered via the coloring $\cttwo$ because
    \begin{equation}\label{eq:distprofiscolor}
        \cttwo = (\ctone, \lms (\ctoneni(i, k), \ctoneni(k, j))  \rms_{k=1}^{n})
    \end{equation}

Thus for every $y\in S$, we can recover the tuple of distances $(d(b, y), d(x_i, y), d(x_j, y))$, by Lemma 3.7, because we also know the multiset $\lms \ctoneni(k,l) \m k, l \in [n]\rms$, as a multiset. This is precisely the definition of the enhanced profile of the desired couple.
\end{proof}

We now present an abbreviated reformulation of the original Reconstruction Algorithm (Section 3.1, \cite{rose2023iterations}), that allows us to recover the point cloud from the enhanced profile of a three-tuple of points the satisfy the cone condition:

\paragraph{Reconstruction Algorithm.}\cite{rose2023iterations} For any triplet $\mathbf{x}=(x_i, x_j, x_k) \in S^3$ that satisfies the \textit{Cone Condition} w.r.t $S$ we can recover the point cloud $S$ from the \textit{Enhanced Profile} of $\mathbf{x}$.

We now show how for every pair $(i,j)\in [n]$ we extract a three-tuple $(x_i, x_k, x_l)$ for some $k, l \in [n]$ that satisfies the \textit{Cone condition}, and also recover the tuple's \textit{Enhanced Profile}, from the fourth update step of $2$-WL, which is defined as
\begin{equation}\label{eq:2-wl-update-4}
\mathbf{C_{(4)}}(i,j) = \embed(\mathbf{C_{(3)}}(i,j), \lms  N_{k}(\mathbf{C_{(\mathbf{3})}}(i,j))  \rms_{k=1}^{n}  ),
\end{equation}

where
\begin{equation}\label{eq:updateapp}
    N_{k}(\mathbf{C_{(3)}}(i,j)) = \left(\mathbf{C_{(\mathbf{3})}}(i,k), \mathbf{C_{(\mathbf{3})}}(k,j) \right)
\end{equation}

We unpack the aggregation multi-set : 

\begin{align}\label{unpack}
     \mathbf{C_{(\mathbf{4})}}(i,j) &= \lb \mathbf{C_{(\mathbf{3})}}(i,j), \lms  N_{k}(\mathbf{C_{(\mathbf{3})}}(i,j))  \rms_{k=1}^{n} \rb = \lms \left(\mathbf{C_{(\mathbf{3})}}(i,j), \mathbf{C_{(\mathbf{3})}}(i,k), \mathbf{C_{(\mathbf{3})}}(k,j) \right) \; | \; k\in [n] \rms \\ \label{eq:recover}
    &\supseteq  \lms ( (\mathbf{C_{(\mathbf{2})}}(i,j), \mathbf{C_{(2)}}(i,k), \lms (\mathbf{C_{(\mathbf{2})}}(i,l), \mathbf{C_{(\mathbf{2})}}(l,k)) \; | \; l \in [n] \rms ),  \\
    &  \; \; \; \; \; \; \; \; \; (\mathbf{C_{(\mathbf{2})}}(i,j),\mathbf{C_{(2)}}(k,j), \lms (\mathbf{C_{(\mathbf{2})}}(k,s), \mathbf{C_{(\mathbf{2})}}(s,j)) \; | \; s \in [n] \rms ) \; | \; k\in [n] \rms \nonumber
\end{align}

where $\supseteq$ denotes we can recover the information of the smaller multiset (with respect to the inclusion relation) from the larger multiset.

If point cloud satisfies AffineDim(S)$\leq 2$ then there are three options: 

\begin{enumerate}
    \item the point cloud is multiple copies of the barycenter
    \item the points in the point cloud are co-linear w.r.t the barycenter ( i.e. $\affdim(S)=1$)
    \item the point cloud lies in a two-dimensional hyperplane
\end{enumerate}

In the first two cases, for a fixed pair $(i, j) \in [n]^2$, we can recover the entire point cloud from the distances of $x_i, x_j, b$ to the other points, see Lemma 3.4 \cite{rose2023iterations}. Note that his information is precisely the distance profile of $(b,x_i,x_j$. We can recover this information from $\mathbf{C_{(\mathbf{2})}}(i,j)$ and $\lms \mathbf{C_{(\mathbf{1})}}(k,l) | l\in [n], k\in [n] \rms$, see \cite{rose2023iterations} Lemma 3.8. These colorings are known via \cref{eq:recover} (we can recover any coloring from its respective successive coloring). Specifically, $\mathbf{C_{(\mathbf{2})}}(i,j)$ can be derived from $\mathbf{C_{(\mathbf{3})}}(i,j)$, see \cref{unpack}, and $\lms \mathbf{C_{(\mathbf{1})}}(k,l) | l\in [n], k\in [n] \rms$ can be recovered from \cref{eq:recover} by extracting $\mathbf{C_{(\mathbf{2})}}(k,l)$ for every $k,l\in [n]$ as  a multiset, then this yields $\lms \mathbf{C_{(\mathbf{1})}}(k,l) \m k,l \in [n]\rms$ as desired.

If the points are in a two-dimensional hyperplane then we can recover the point cloud from the distances of all other point clouds onto a pair that span the 2D hyperplane w.r.t to the point cloud's barycenter. If $x_i, x_j$ satisfy this then we are done by Lemma 3.4 in \cite{rose2023iterations}, as we have shown above we can recover the distance profile of the tuple $(b,x_i,x_j)$. If not, and there exists $x_k$ such that $x_i,x_k$ ( and also $x_j, x_k$ bacause $x_i, x_j$ are colinear) satisfy the spanning condition. Otherwise, the point cloud would be colinear and we have already addressed that case. W.l.o.g $x_i, x_k$ span the point cloud w.r.t the barycenter, then we can recover $x_j$ because it is colinear to $x_i$ and simply knowing the distance from $x_j$ to $x_i$ and $b$ fully determines it (Lemma 3.4 \cite{rose2023iterations}), which we can recover from $\mathbf{C_{(\mathbf{2})}}(i,j)$ and $\lms \mathbf{C_{(\mathbf{1})}}(k,l) | l\in [n], k\in [n] \rms$ (see previous paragraph). We can find this $k$ from the multiset in \cref{eq:recover} because we can recover the distance profile of $(b, x_i, x_k)$ from $\mathbf{C_{(\mathbf{2})}}(i,k)$ and $\lms \mathbf{C_{(\mathbf{1})}}(k,l) | l\in [n], k\in [n] \rms$, from which the angle defined by $b, x_i, x_j$ can be computed (because we can recover $b, x_i, x_j$ from its distance profile). All the other points in the point cloud can be recovered from the distance profile of $(b, x_i, x_j)$ as a set (Lemma 3.4 \cite{rose2023iterations}). We are done.

Now that we have addressed the low dimensional cases, we assume for the remainder of the proof that AffineDim(S)$ =  3$. Note that \text{AffineDim}$(b,x_i, x_j)$ is a euclidean invariant feature of $b,x_i,x_j$ thus can be computed from the distance profile of $(b, x_i, x_j)$, because $(b,x_i,x_j)$ can be recovered from it up to euclidean symmetries.

\paragraph{Case 1.} \text{AffineDim}$(b,x_i, x_j)$ = 2.

We assume that both $\affdim(S)=3$ and $\affdim(b, x_i,x_j)=2$, therefore there exists an index $k \in [n]$ such that \text{AffineDim}$(b,x_i, x_j, x_k)$ = 3. Denote by $M$ the set of all such indices, i.e. $L := \{ m \m l \in [n] \text{ and } \affdim(b, x_i,x_j, x_l)=3 \}$. There exists $k\in L$ such that the cone defined by $(b,x_i,x_j,x_k)$ has minimal angle, where the angle of a 3-dimensional cone is defined as 

\begin{equation}\label{def:angle}
    \textrm{Angle}( x_1-b, x_2-b, x_3-b) = \frac{1}{3}\textrm{Vol}(\{ x\in \textrm{Cone}( x_1-b, x_2-b, x_3-b) \m \|x\|\leq 1 \})
\end{equation}

For any $x\in S$ it holds that $x-b \notin \mathrm{Interior}( \cone(b,x_i, x_j, x_a))$, because otherwise there exists an index $h \in [n]$ such that $\mathrm{Angle}(x_i-b,x_j-b,x_h-b) < \mathrm{Angle}(x_i-b,x_j-b,x_k-b)$, see Lemma 3.9  in \cite{rose2023iterations}, in contradiction to $k$ attaining the minimal angle.

To recover the labeled point cloud, we need to show that we can indeed recover the a tuple $(x_i,x_j,x_k)$ obtains this minimum with respect to its respective cone's angle, and recover the Enhanced Profile of $(x_i, x_j, x_k)$, which we denoted as $EP(x_i, x_j,x_k)$. Then we will use the Reconstruction Algorithm introduced by \citeauthor{rose2023iterations}.

\cref{unpack} yields $\lms \left(\mathbf{C_{(\mathbf{3})}}(i,j), \mathbf{C_{(\mathbf{3})}}(i,k), \mathbf{C_{(\mathbf{3})}}(k,j) \right) \; | \; k\in [n] \rms$, we then need to pick a $k$ that minimizes $\mathrm{Angle}(x_i-b,x_j-b,x_k-b)$. As $\mathrm{Angle}(x_i-b,x_j-b,x_k-b)$ is a Euclidean invariant feature then we can recover it from $\left(\mathbf{C_{(\mathbf{3})}}(i,j), \mathbf{C_{(\mathbf{3})}}(i,k), \mathbf{C_{(\mathbf{3})}}(k,j) \right) $, because we can reconstruct $(x_i, x_j, x_k)$ w.r.t the barycenter up to Euclidean actions from $\left(\mathbf{C_{(\mathbf{2})}}(i,j), \mathbf{C_{(\mathbf{2})}}(i,k), \mathbf{C_{(\mathbf{2})}}(k,j) \right) $ and $\lms \mathbf{C_{(1)}}(k,l) \m k,l \in [n] \rms$ by Lemma 3.8. Thus, we can choose an index $k$ that minimizes this angle from the multiset. (We can recover the multiset from \cref{eq:recover} as mentioned earlier.)

By the Corollary, we can recover the Enhanced Profile of the tuple $(x_i,x_j,x_k)$. Now that we have recovered $EP(x_i,x_j,x_k)$ for an index $k$ that yields the minimal cone angle, we use the Reconstruction Algorithm (Section 3.1 \cite{rose2023iterations})
 in order to recover the point cloud. Note that the index $k$ is known only to exist in the multiset, thus cannot be labeled, yet the indices $i,j$ can be labeled as we first reconstructed them from an ordered tuple and then recovered the rest of the points as a multiset.
\paragraph{Case 2.} \text{AffineDim}$(b,x_i, x_j) =1$.

We need only recover w.l.og (because $x_i$ and $x_j$ are co-linear) the enhanced profile of $(x_i,x_l,x_k)$ that satisfies the cone condition and then knowing $(\|b-x_j\|\|x_i-x_j\|)$ we can recover $x_j$ uniqely (as $x_i$ and $x_j$ are co-linear,see Lemma 3.4 \cite{rose2023iterations}). Below we show we can indeed recover this information from $\ctfour$.

First, from \ref{eq:recover}, we can recover $\mathbf{C_{(1)}}(i,j)$ and the multiset $\lms \ctoneni(l,k) \m l,k  \in [n] \rms$, then by Lemma 3.7, we know the distance $\|x_j-b\|$ and from $\mathbf{C_{(0)}}(i,j)$ we know $\|x_i-x_j\|$\label{distfromi}.

We now turn to recovering $EP(x_i,x_l,x_k)$ such that \text{AffineDim}$(x_i, x_l, x_k) = \text{AffineDim}(S)$.

We first show the existence of $x_l, x_k$ that satisfy above condition.

We know that exists $k$ such that $\text{AffineDim}(b,x_i,x_k)>1$, otherwise $\text{AffineDim}(S)=1$. Contradiction.

Similarly, if $\text{AffineDim}(b,x_i,x_k,x_l)=2$ for all $l$, then 
$\text{AffineDim}(S)=2$. Contradiction.

Thus there exist $i,k,l$ that satisfy \text{AffineDim}$(x_i, x_l, x_k) = \text{AffineDim}(S)$. We now consider the $k, l$ that minimize the angle of $Cone(x_i-b, x_k -b, x_l-b)$. There exist no points in the interior of this chosen cone, see Lemma 3.8 \cite{rose2023iterations} and explanation in the proof of Case 1.

We can extract the Enhanced Profile of the tuple $(x_i,x_l,x_k)$ that satisfies the cone condition from \cref{eq:recover}, i.e. $\lms (\mathbf{C_{(2)}}(i,k), \lms (\mathbf{C_{(\mathbf{2})}}(i,l), \mathbf{C_{(\mathbf{2})}}(l,k))  | \; l \in [n] \rms ) \; k\in [n] \rms $. (We explained in Case 1 why this can be performed)

Using the Reconstruction Algorithm, we can recover the tuple $(x_i, \lms x_l \; \mid \; l\in [n]\rms)$ and, as we have shown previously, we can recover (a labeled) $x_j$, yielding a final reconstruction, up to Euclidean symmetries, of the desideratum tuple $(x_i, x_j, \lms x_l \; \mid \; l\in [n]\rms)$. 

We now address the (simple) case that $x_i=x_j=b$ (i.e. $\affdim(b,x_i,x_j)=0$). From the coloring $\ctfive$ we can recover the multiset $\lms \ctthreeni(k,l) \m k,l \in [n] \rms$,by definition of the update step of $2$-WL and unpacking the coloring $\ctfive$ analogously to what we have done in \cref{unpack}, and this is precisely the information required by the original Reconstruction Algorithm devised in \cite{rose2023iterations}(Theorem 1.1.) Finally, as each refined coloring contains all the information of the previous colorings $\ctfive$ is sufficient to recover the tuple $(x_i,x_j, \lms x_k \m k\in [n]\rms$ up to Euclidean symmetries for \textit{any} point cloud in $\RR^{3}$.
\end{proof}

\newpage
\newpage
\twoclouds*
\begin{proof}
    We assume that $X$ respects all Euclidean and permutation symmetries and $V$ does as well, with the exception that it is \textit{not} translation invariant. 
We take the complete distance matrix of $(X,V)$ after centralizing $X$ where we allow for norms for the velocity vectors.
    First, we introduce notation, definitions and lemmas which will be useful later for the proof:

    Denote $X\cup V$ as the multiset $\lms y \m y\in X \text{ or } y\in V \rms$, $\mathrm{Concat}(X,V)$ as the point cloud in $\RR^{n \times 6}$ which is the concatenation of $X$ and $V$ along the feature dimension, and $X\times V$ as the multiset $\lms (x_i,v_i) \m (x_i,v_i) \in \textrm{Concat}(X,V) \rms$.
    
    Define the \textbf{origin distance profile} of $\mathbf{y}$ w.r.t. $X\times V$ as the multiset
\begin{equation}\label{def:origindistprofile}
   D_{\mathbf{x}} = \lms (d(\zv,v) , (d(x_1,x), d(x_1,v)),  \ldots, (d(x_k, x), d(x_k, v)) )  \mid (x,v) \in X \times V\rms 
\end{equation}

\begin{definition}\label{def:enhancedproforigin}
For a tuple $y = (y_1, y_2, y_3) \in (x_1, v_1) \times (x_2, v_2)  \times (x_3, v_3)) \in X\cup V$, we define its origin enhanced profile as
\[EP(y_1, y_2, y_3) = (A, M_1,M_2, M_3),\]
where $A$ is the distance matrix of the tuple $(b, x_1, v_1 \dots, x_3, v_3)$ with the norms of the $v_i$'s on the digonal, and $M_i = D_x[b/i]$ is the origin distance profile (see Equation \ref{def:origindistprofile}) of the tuple $(x_1, v_1, \dots, x_{i - 1}, x_{i - 1}, \bv,\bv, x_{i + 1},v_{i + 1}, \dots, x_3, v_3)$ with respect to $X\times V$.
\end{definition}

    \paragraph{Lemma 3.4 } \cite{rose2023iterations} Given a set of points $\lms x_1, \ldots, x_m \rms \in \RR^{3}$ and a point $y\in \RR^{3}$ such that $y \in \mathrm{AffineSpan}(x_1, \ldots, x_m)$ then the multiset $\lms d(y, x_i)\rms_{i=1}^{m}$ uniquely determines $y$. If $y\notin \textrm{AffineSpan}(x_1,\ldots,x_m)$, then we can recover $y$ up to orthogonal actions w.r.t $\lms x_1, \ldots, x_m \rms$.

    \paragraph{Lemma 1.} For a multiset $\lms \ctokl \m k,l \in [n]\rms $, we can recover the origin, $\zv$, w.r.t the barycenter of $V$, $b_{V}$. If also given $\ctone$, we can recover the distance matrix of of the tuple $(\bv, \zv, v_i, v_j)$.
    \begin{proof}[Proof of Lemma 1]
        By the Barycenter Lemma (2.1 \cite{rose2023iterations}) we can recover the distance  $\|\Vec{0}_{V}-b_{V}\|$ if we know the multisets $D_{\zv} := \lms \|\Vec{0}_{V} - v \| \m v \in V \rms$ and $\lms D_{v} \m v\in V\cup \{ \zv \} \rms$. Note that adding zero elements to the set $D_{\zv}$ still allows recovery of $\|\Vec{0}_{V}-b_{V}\|$, because in the proof of the barycenter lemma it is only required that we recover $\sum_{y\in V} \| \zv - v \|$ which is not altered by zero elements, once we know the number of elements in the point cloud, see proof of Lemma 2.1 in \cite{rose2023iterations}.

        We now show how to extract this information from the initial distance matrix of $V$ and the norms on the diagonal. From the initial coloring, we know whether a color $\mathbf{C_{0}}{(i,j)}$ satisfies $x_i=x_j$  if the off-diagonal colorings are $0$. This does not necessarily imply that $i=j$, but it is sufficient for our ends.

        We can recover the multiset $M:=\lms \czplain(k,l) \m k,l\in [n] \; \textrm{s.t.}\; \| x_k - x_l \| =0\rms$ from which the multi-set $\lms \czplain(k,l)[1,1] \m \czplain(k,l) \in M  \rms$ can be extracted. We have a multiset of the norms of all of the vectors in $V$ and we might have added superfluous zeros which are permissible, see the above paragraph for an explanation. These norms are the distance from $\zv$ to the rest of the points, as required. We know the number of points in the velocity point cloud because we know there exist $n^2$ elements in the multiset $\lms  \ctoneni(k,l) \rms_{k,l=1}^{n}$
        Recovering the multiset $\lms D_{v} \m v\in V\cup \{ \zv \} \rms$ is straightforward as recovering each $D_{v}$ for $v\in V$ can be done via Lemma 3.7 as in \cite{rose2023iterations} and we have shown how to recover $D_{\zv}$

        If we additionally know $\ctone$ then we can recover the distances of $v_i,v_j$ to $\zv$. thus using the above arguments and Lemma 3.7\cite{rose2023iterations}, we can recover the desired distance matrix.
        We are done.
    \end{proof}
    \paragraph{Lemma 2.} There exists a triplet $(y_1,y_2, y_3)$ where $y_i \in X$ or $y_i \in V$ for any $i\in [3]$ such that $\affdim( b_{V}, y_1, y_2, y_3) = \affdim{( X \cup V )}$. 
    \begin{proof}[Proof of Lemma 2]
        Let $M:= \affdim{( X \cup V )}$. Assume by contradiction that for any such triplet $\affdim(\bv, y_1, y_2,y_3) < M$ then any three vectors lie in the same hyperplane of deficient dimension, which implies $\affdim{( X \cup V )} < M$. Contradiction. 

    \end{proof}

    Assume that  that $\affdim{( X \cup V)}=3$. we will subsequently address the general case.

    We now attempt to reconstruct the point cloud w.r.t the barycenter of $V$, denoted by $b_{V}$, with via the multiset $\mathbf{C}_{G} := \lms \ctthreeni(j,k) \rms_{j,k=1}^{n}$.

        By iterating over all tuples in $X\times V$ and calculating $\affdim(b_{V}, y_1, y_2, y_3)$ for any $(y_1,y_2,y_3)\in (x_1, v_1) \times (x_2, v_2) \times (x_3, v_3)) \in X\times V$  and calculating the angle of the cone defined by $(b_{V}, y_1, y_2, y_3)$, with analogous justification to the the proof of \cref{lem:w-wl-edge} (these are Euclidean invariant functions). Thus, by assumption and by analogous justifications as in the proof of \cref{lem:w-wl-edge},  
            there exists a triplet $\mathbf{v}=(v_1, v_2, v_3) \in V$ such that $\affdim(\mathbf{v})=\affdim{( X \cup V)}$ and that no other points in $X\cup V$ are contained within its respective cone, this is there exists a triplet $(v_1,v_2,v_3) \in V^{3}$ that satisfies the Cone Condition. All that is left is to recover the Origin Enhanced Profile, see \cref{def:enhancedproforigin}, of this triplet from $\mathbf{C}_{G}$. It is readily seen that we can recover the Origin Enhanced Profile as the proof of doing so by \cite{rose2023iterations} relies on Lemmas 3.7 and 3.8, which via Lemma 1 and the definition of the pair-wise colorings $\cttwoni(i,j)$ is immideatly extended to include the distances of the $v$'s to the origin $\zv$, and the correspondance between $x_i$ and $v_i$ which is hard encoded into the pair-wise-colorings, as supplementaries to the original distance profile, see \cref{def:distprofile}, which precisely yields the Origin Enhanced Profile.

    If w.l.o.g there exists a tuple $((x_1, v_1),(x_2, v_2),(x_3, v_3))\in X\times V$ such that $(v_1,v_2,v_3) $ satisfies the cone condition ( i.e. $\affdim(b,v_1,v_2,v_3)=\affdim(X \cup V)$ and no other point in $X\cup V$ is contained in the cone defined by this triplet) then we can run the Reconstruction Algorithm with the Origin Enhanced Profile to recover $X$ and $V$ such that we can recover $\zv$, as the triplet $(v_1,v_2,v_3)$ spans $\RR^{3}$ and we know the distances of each $v_1,v_2,v_3$ to $\zv$, by definition of the Origin Distance Profile.

    Otherwise, there exists at least an $x\in X$ in the triplet which satisfies the cone condition.  Using Lemma 1, we can run the reconstruction algorithm via the information $\mathbf{C}_{G}$ with the origin distance profile, due to the derivation of the distance profile from the information in $\cttwo$ and Lemma 1. Note the unpacking of $\cttwo$ is defined as:

    \begin{equation}\label{eq:recover-origin}
        \cttwo = (\ctone, \lms \ctoneni(i,k), \ctoneni(k,j) \rms_{k=1}^{n})
    \end{equation}
    
     Thus, by the procedure of the Reconstruction Algorithm, see \cite{rose2023iterations}, and the definition of the Origin Enhanced Profile along the Reconstruction Algorithm for every recovered $v_i$ we can recover its distance from the origin of $V$, i.e. $\zv$. Thus, we can recover $\lms (v_i, \|v_i - \zv \|) \rms_{i=1}^{n}$, because for each recovered $v \in V$  we know its norm, which is the  distance from the origin of $V$ to $v$, and the distance of the origin from the barycenter of $V$. Thus, we can recover the origin w.r.t V, uniquely, see Lemma 3.4, because by assumption $\zv \in \mathrm{AffineSpan(V)}$.

     Otherwise, if $\affdim(X\cup V)<3$, then there does not exist such a triplet that satisfies the cone condition, yet there exists a tuple that spans the point cloud $X\cup V$, by Lemma 2. Thus, we can reconstruct the point cloud w.r.t the barycenter of $V$ in an analogous fashion to the low dimensional cases in the proof of \cref{lem:w-wl-edge}, but without the labeling of $i$,$j$, via iterating over all the modified pair tuples as in the above case where $\affdim(X\cup V)=3$. Due to Lemma 3.4, we can recover the origin w.r.t V, i.e. $\zv$, uniquely if $\zv \in \mathrm{AffineSpan}(y_1,y_2,y_3)$ or up to orthogonal action, otherwise.





\end{proof}

\newpage
\begin{restatable}{corollary}{twowledgetwopoints}\label{lem:w-wl-edge-two-points}
	Let $(X,V)\in (\RR^{3 \times n})^2$ be a pair of position and velocity point clouds. Denote by $ \ctfive $ the pairwise coloring induced by 5 iterations of 2-WL update steps applied to the distance matrix induced by $(X, V)$. Then $\ctfive$ constitutes a translation, rotation and reflection invariant embedding of $((x_i, v_i), (x_j, v_j), \lms (x_k, v_k)  \m k\neq i,j  \rms)$. If $(x_i, v_i), (x_j, v_j)$ are not both degenerate, i.e. do not both equal the barycenter, then  $\ctfour$ is sufficient.
\end{restatable}
\begin{proof}
    In the proof of \cref{lem:w-wl-edge}, the analogous result for a single point cloud, we have chosen a triplet $(x_i,x_j,x_k)$ that satisfies the cone condition. In this Lemma, we choose a triplet $\mathbf{xv}:=((x_i, v_i),(x_j,v_j),(x_k,v_k))$ such that there exists a triplet $(y_i,y_j,y_k) \in \mathbf{xv}$ that satisfies the cone condition, where each $y_i$ can be either in $X$ or $V$. Combining this with \cref{thm:pairs}, the desideratum follows. Notation used in this proof was introduced in the proof of \cref{thm:pairs}.

    \paragraph{Existance.} For a candidate $\mathbf{xv}$ we choose $(y_i,y_j,y_k) \in \mathbf{xv}$ such that they minimize $\mathrm{Angle}(y_i - \bv,y_j- \bv,y_k- \bv)$. If we fix the indices $i,j$ and iterate over all such $k$, then we can find a triplet that satisfies the cone condition or is of equivalent dimensionality to that of $X \cup V$, see proof of \cref{lem:w-wl-edge}. We can then run the Reconstruction Algorithm as in the proof of \cref{thm:two-clouds}. Note that in this setting we first recover $\mathbf{xv}$, thus we know the labeled $((x_i, v_i), (x_j, v_j), (x_k, v_k))$ and reconstruct the remainder of the point cloud w.r.t $(\bv, y_i, y_j, y_k)$. 
    \paragraph{Initialization.} Using the same information as in \cref{lem:w-wl-edge}, due to \cref{thm:two-clouds} the \cref{eq:recover} contains pairwise features of pairs of point cloud, thus with the modification described in the \textbf{Existance} paragraph, we can analogously recover the point cloud with labeling of the $i,j$ positions and velocities via the triplet satisfying the cone condition. Thus we have the desideratum embedding $((x_i, v_i), (x_j, v_j), \lms (x_k, v_k) \rms_{k\neq i,j})$.
\end{proof}

\newpage

\universaltwo*
\begin{proof}
We now combine \cref{lem:orthovel} and \cref{lem:w-wl-edge} to obtain the desideratum universal approximation result. Note that any invariant function $h$ that is invariant to a group $G$ can be written as $f=g\circ i$ where $i$ is the quotient map, i.e. if $x = gy$ for some element $g\in G$ then $i(x)=i(y)$ \cite{munkres2000topology}. Thus we can reformulate the result in \cref{lem:orthovel}, for $f$ invariant to permutations of the last $n-2$ coordinates, rotation, and translation, 
. we can write as $f = h \circ \text{Embed}(x_i,x_j, \lms x_k \rms)$ where $\text{Embed}(x_i,x_j, \lms x_k \rms)$ is \textit{injective} up to Euclidean symmetry and invariant to permutations of the last $n-2$ coordinates, rotation, and translations and $h$ is continuous. Using \cref{lem:w-wl-edge}, we can use the colorings derived from $5$ iterations of $2$-WL in order to obtain $\text{Embed}(x_i,x_j, \lms x_k \rms)$, i.e. the colorings $\bc_{(5)}(i,j)$.

Further combining with Proposition 7 \cite{villar}, for any $i \in [n]$, 
\begin{align}
    x_{i}^{\mathrm{out}} = x_{i} &+ f((x_i,v_i), \lms \br{x_k,v_k} \m k \neq i \rms ) v_i \\&+ \sum_{j} \phi((x_i,v_i), (x_j, v_j), \lms \br{x_k,v_k} \m k \neq i,j \rms)(x_j - x_i) \nonumber\\
    & + \sum_{j} \hat{\phi}((x_i,v_i), (x_j, v_j), \lms \br{x_k,v_k} \m k \neq i, j \rms)v_k\nonumber \\ 
    =x_{i} &+ g (\text{Embed}((x_i,v_i), \lms \br{x_k,v_k} \m k \neq i \rms))v_i \\ &+ \sum_{j} \psi(\text{Embed}((x_i,v_i), (x_j, v_j), \lms \br{x_k,v_k} \m k \neq i, j \rms ))(x_j - x_i) \nonumber\\
    & + \sum_{j} \hat{\psi}(\text{Embed}((x_i,v_i), (x_j, v_j),\lms \br{x_k,v_k} \m k \neq i, j \rms )) v_j\nonumber \\ 
    =x_{i} &+ g(\bc_{(5)}(i,i))v_i\\&+ \sum_{j} \psi(\bc_{(5)}(j,i))(x_j - x_i) \nonumber\\
    & + \sum_{j} \hat{\psi}(\bc_{(5)}(j,i))v_j\nonumber  
\end{align}
where $\sum_{j} \psi(\bc_{(5)}(j,i)) = 1 $ and for invariant continuous $f, \phi$, and  continuous (not invariant) $g, \psi$. Note that the tuple $((x_i,v_i), (x_j, v_j), \lms \br{x_k,v_k} \m k \neq i,j \rms)$ can be recovered from $((x_i,v_i), (x_j, v_j)\lms \br{x_k,v_k} \m k \neq i, j \rms )$ up to Euclidean symmetries.

This precisely yields \cref{eq:eqpoolx} for $\sum_{j} \psi(\bc_{(T)}(i,j)) = 1 $. The velocity update is defined analogously but with different $g, \psi, \hat{\psi}$ and it is translation invariant, thus also by Proposition 7 \cite{villar}, $\sum_{j} \psi(\bc_{(T)}(i,j)) = 0 $, yielding

\begin{equation}\label{eq:update-v}
    v_{i}^{out} = g'(\bc_{(5)}(i,i))v_i+ \sum_{j} \psi'(\bc_{(5)}(j,i))(x_j - x_i)
    + \sum_{j} \hat{\psi}'(\bc_{(5)}(j,i))v_j    
\end{equation}

We can now approximate the $\psi$'s and $g$'s via MLPs \cite{cybenko1989approximation} yielding an approximation of all equivariant polynomials. Equivariant polynomials are dense in equivariant continuous functions \cite{Dym2020OnTU}, thus (by the triangle inequality) we can approximate all continuous equivariant functions on $(X,V)$ via the pooling operator, as required.
\end{proof}

\end{document}